\documentclass{article}
\pdfoutput=1
\usepackage{microtype}
\usepackage{graphicx}
\usepackage{subfigure}
\usepackage{booktabs} 
\usepackage{stmaryrd}
\usepackage{amsfonts}
\usepackage{amsthm}
\usepackage{amssymb}
\usepackage{mathrsfs}
\usepackage[cmex10]{amsmath}
\usepackage{amssymb}
\usepackage{mathtools}
\usepackage{thmtools,thm-restate}
\usepackage{verbatim}
\usepackage{enumitem}
\usepackage{mathabx}
\newtheorem{defn}{Definition}
\newtheorem{assum}{Assumption}

\newcommand{\norm}[1]{\left \| #1 \right \|}

\newcommand{\normD}[1]{\left \| {#1} \right \|_\D}

\newcommand{\wass}[2]{W\left(#1,#2\right)}
\newcommand{\wassd}[2]{W_d\left(#1,#2\right)}
\newcommand{\W}{W}
\newcommand{\D}{\mathcal{D}}
\newcommand{\Dist}[2]{\mathcal{D}\left(#1,#2\right)}

\newcommand{\F}{\mathcal{F}}

\newcommand{\mdp}{\mathcal{M}}
\newcommand{\sspace}{\mathcal{S}}
\newcommand{\rew}{\mathcal{R}}
\newcommand{\prob}{\mathcal{P}}

\newcommand{\aspace}{\mathcal{A}}

\newcommand{\dmdp}{\widebar{\mathcal{M}}}
\newcommand{\dsspace}{\widebar{\mathcal{S}}}
\newcommand{\drew}{\widebar{\mathcal{R}}}
\newcommand{\dprob}{\widebar{\mathcal{P}}}

\newcommand{\tPi}{\widetilde{\Pi}}
\newcommand{\tpi}{\widetilde{\pi}}
\newcommand{\td}{\widetilde{d}}

\newcommand{\Krew}{K_{\rew}}
\newcommand{\Kprob}{K_{\prob}}
\newcommand{\Kdrew}{K_{\drew}}
\newcommand{\Kdprob}{K_{\dprob}}
\newcommand{\KdV}{K_{\dV}}
\newcommand{\KV}{K_{V}}

\newcommand{\KdQ}{K_{\dQ}}

\newcommand{\dists}{d_{\sspace}}
\newcommand{\distds}{d_{\dsspace}}

\newcommand{\ds}{\widebar{s}}
\newcommand{\dpi}{\widebar{\pi}}
\newcommand{\xipi}{\xi_\pi}

\newcommand{\xidpi}{\xi_{\dpi}}

\newcommand{\dPi}{\widebar{\Pi}}

\newcommand{\dV}{\widebar{V}}
\newcommand{\dVpi}{\widebar{V}^{\dpi}}
\newcommand{\Vpi}{V^{\pi}}
\newcommand{\Vdpi}{V^{\dpi}}
\newcommand{\dVdpi}{\dV^{\dpi}}

\newcommand{\Qdpi}{Q^{\dpi}}
\newcommand{\dQ}{\widebar{Q}}
\newcommand{\dQdpi}{\widebar{Q}^{\dpi}}
\newcommand{\dQpi}{\widebar{Q}^{\dpi}}

\newcommand{\pistar}{\pi^*}
\newcommand{\dpistar}{\dpi^*}
\newcommand{\Vstar}{V^*}
\newcommand{\dVstar}{\widebar{V}^*}
\newcommand{\Qstar}{Q^*}
\newcommand{\dQstar}{\widebar{Q}^*}

\newcommand{\Ldrew}{L_{\drew}}

\newcommand{\Ldprob}{L_{\dprob}}

\newcommand{\gLdprob}{\Ldprob^\infty}
\newcommand{\gLdrew}{\Ldrew^\infty}
\newcommand{\lLdprob}{ \Ldprob^{\xidpi}}
\newcommand{\lLdrew}{\Ldrew^{\xidpi}}
\newcommand{\lLdprobxi}{ \Ldprob^{\xi}}
\newcommand{\lLdrewxi}{\Ldrew^{\xi}}

\DeclareMathOperator*{\expect}{\mathbb{E}}

\newcommand{\expxi}{{\expect_{\substack{s, a \sim \xi}}}}

\newcommand{\expxidpi}{{\expect_{\substack{s, a \sim \xidpi}}}}

\newcommand{\expddpi}{ \expect_{s \sim \xi_{\dpi}} }

\newcommand{\deriv}{~d}

\newtheorem{lemma}{Lemma}

\theoremstyle{definition}

\usepackage{hyperref}

\usepackage[accepted]{icml2019}

\icmltitlerunning{DeepMDP: Learning Continuous Latent Space Models for Representation Learning}

\begin{document}

\twocolumn[
\icmltitle{DeepMDP: Learning Continuous Latent Space Models \\ for Representation Learning}


\icmlsetsymbol{equal}{*}

\begin{icmlauthorlist}
\icmlauthor{Carles Gelada}{go}
\icmlauthor{Saurabh Kumar}{go}
\icmlauthor{Jacob Buckman}{jh}
\icmlauthor{Ofir Nachum}{go}
\icmlauthor{Marc G. Bellemare}{go}
\end{icmlauthorlist}

\icmlaffiliation{go}{Google Brain}
\icmlaffiliation{jh}{Center for Language and Speech Processing, Johns Hopkins University}

\icmlcorrespondingauthor{Carles Gelada}{cgel@google.com}

\icmlkeywords{Machine Learning, ICML}

\vskip 0.3in
]



\printAffiliationsAndNotice{}  

\begin{abstract}
Many reinforcement learning (RL) tasks provide the agent with high-dimensional observations that can be simplified into low-dimensional continuous states.
To formalize this process, we introduce the concept of a \textit{DeepMDP}, a parameterized latent space model that is trained via the minimization of two tractable losses: prediction of rewards and prediction of the distribution over next latent states.
We show that the optimization of these objectives guarantees (1) the quality of the latent space as a representation of the state space and (2) the quality of the DeepMDP as a model of the environment.
We connect these results to prior work in the bisimulation literature, and explore the use of a variety of metrics.
Our theoretical findings are substantiated by the experimental result that a trained DeepMDP recovers the latent structure underlying high-dimensional observations on a synthetic environment.
Finally, we show that learning a DeepMDP as an auxiliary task in the Atari 2600 domain leads to large performance improvements over model-free RL.
\end{abstract}

\section{Introduction}
\label{introduction}


In reinforcement learning (RL), it is typical to model the environment as a Markov Decision Process (MDP). However, for many practical tasks, the state representations of these MDPs include a large amount of redundant information and task-irrelevant noise. For example, image observations from the Arcade Learning Environment \cite{bellemare13arcade} consist of 33,600-dimensional pixel arrays, yet it is intuitively clear that there exist lower-dimensional approximate representations for all games. Consider \textsc{Pong}; observing only the positions and velocities of the three objects in the frame is enough to play. Converting each frame into such a simplified state before learning a policy facilitates the learning process by reducing the redundant and irrelevant information presented to the agent. Representation learning techniques for reinforcement learning seek to improve the learning efficiency of existing RL algorithms by doing exactly this: learning a mapping from states to simplified states.


Prior work on representation learning, such as state aggregation with bisimulation metrics \citep{givan2003equivalence,ferns2004bisimulation,ferns2011bisimulation} or feature discovery algorithms \citep{Comanici2011BasisFD, Mahadevan2007ProtovalueFA, Bellemare2019AGP}, has resulted in algorithms with good theoretical properties; however, these algorithms do not scale to large scale problems or are not easily combined with deep learning. On the other hand, many recently-proposed approaches to representation learning via deep learning have strong empirical results on complex domains, but lack formal guarantees \citep{jaderberg2016reinforcement, Oord2018RepresentationLW, Fedus2019HyperbolicDA}. In this work, we propose an approach to representation learning that unifies the desirable aspects of both of these categories: a deep-learning-friendly approach with theoretical guarantees.


We describe the \textit{DeepMDP}, a latent space model of an MDP
which has been trained to minimize two tractable losses: predicting the rewards and predicting the distribution of next latent states. DeepMDPs can be viewed as a formalization of recent works which use neural networks to learn latent space models of the environment~\cite{ha2018recurrent,oh2017value,hafner2018learning,lavet2018crar},
because the value functions in the DeepMDP are guaranteed to be good approximations of value functions in the original task MDP. To provide this guarantee, careful consideration of the metric between distribution is necessary. A novel analysis of Maximum Mean Discrepancy (MMD) metrics \citep{Gretton2012AKT} defined via a function norm allows us to provide such guarantees; this includes the Total Variation, the Wasserstein and Energy metrics. These results  represent a promising first step towards principled latent-space model-based RL algorithms.

\begin{figure}
    \centering
    \includegraphics[keepaspectratio, width=.35\textwidth]{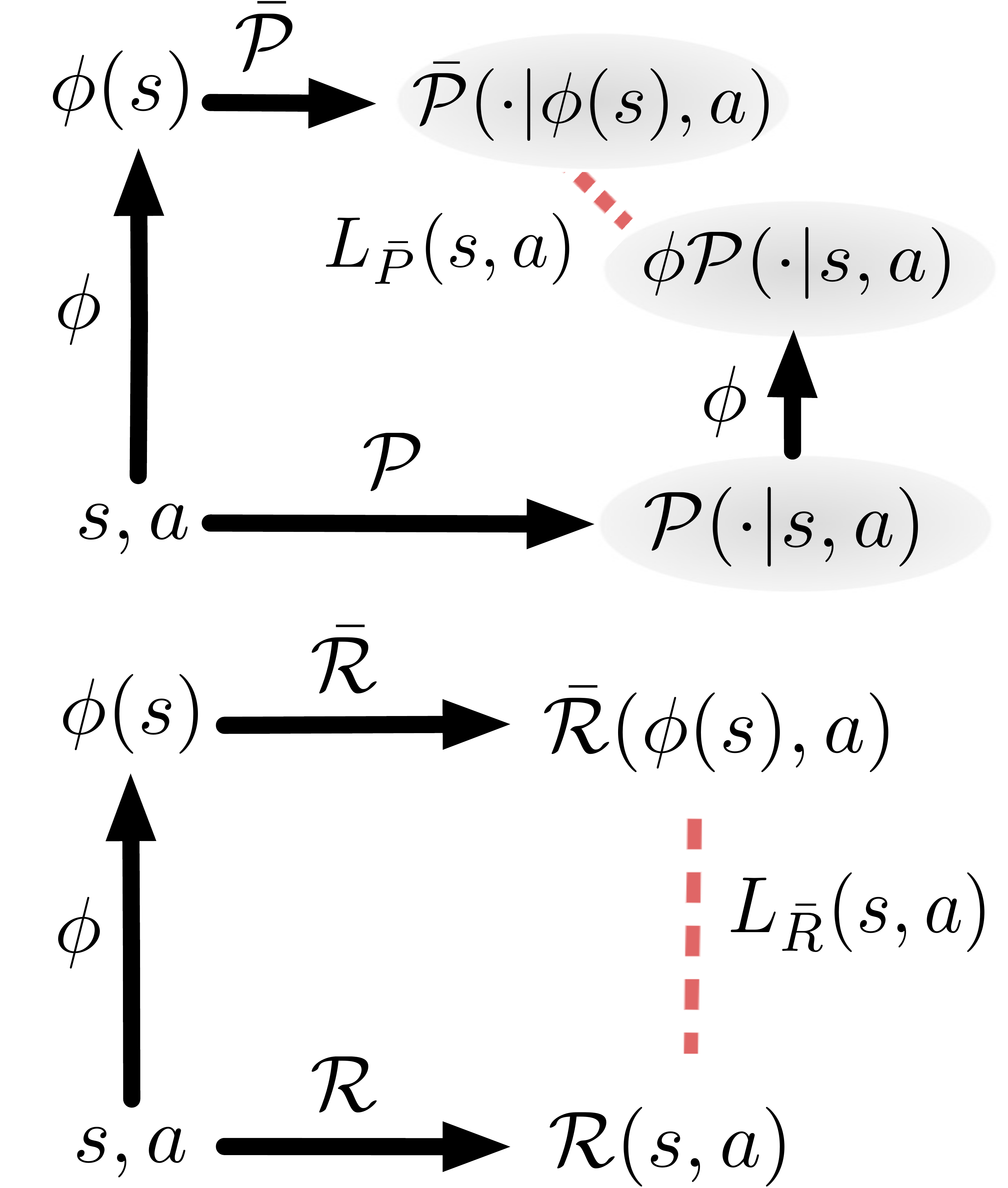}
    \caption{Diagram of the latent space losses. Circles denote a distribution.}
    \label{fig:dmdp_models}
\end{figure}

From the perspective of representation learning, the state of a DeepMDP can be interpreted as a representation of the original MDP's state. When the Wasserstein metric is used for the latent transition loss, analysis reveals a profound theoretical connection between DeepMDPs and bisimulation.
These results provide a theoretically-grounded approach to representation learning that is salable and compatible with modern deep networks.


In Section~\ref{sec:background}, we review key concepts and formally define the DeepMDP.
We start by studying the model-quality and representation-quality results of DeepMDPs (using the Wasserstein metric) in Sections~\ref{sec:global}~and~\ref{sec:local}. In
Section~\ref{sec:bisimulation}, we investigate the connection between DeepMDPs using the Wasserstein and bisimulation.
Section~\ref{sec:other_metrics} generalizes \textit{only} our model-based guarantees to metrics other than the Wasserstein; this limitation emphasizes the special role of that the Wasserstein metric plays in learning good representations.
Finally, in Section~\ref{sec:empirical} we consider a synthetic environment with high-dimensional observations and show that a DeepMDP learns to recover its underlying low-dimensional latent structure. We then demonstrate that learning a DeepMDP as an auxiliary task to model-free RL in the Atari 2600 environment leads to significant improvement in performance when compared to a baseline model-free method.

\section{Background}
\label{sec:background}
\subsection{Markov Decision Processes}
Define a Markov Decision Process (MDP) in standard fashion: $\mdp = \langle\sspace, \aspace, \rew, \prob, \gamma\rangle$~\citep{puterman1994markov}. For simplicity of notation we will assume that $\sspace$ and $\aspace$ are discrete spaces unless otherwise stated. A policy $\pi$ defines a distribution over actions conditioned on the state, $\pi(a|s)$. Denote by $\Pi$ the set of all stationary policies.
The value function of a policy $\pi\in\Pi$ at a state $s$ is the expected sum of future discounted rewards by running the policy from that state. $\Vpi: \sspace \to \mathbb{R}$ is defined as:
\begin{equation*}
V^\pi(s) = \expect_{\substack{a_t\sim\pi(\cdot | s_t) \\ s_{t+1} \sim \prob(\cdot | s_t, a_t)}}\left[\sum_{t=0}^\infty \gamma^t \rew(s_t, a_t) |s_0 = s\right].
\end{equation*}
The action value function is similarly defined:
\begin{equation*}\hspace{-4pt}
Q^\pi(s, a) = \expect_{\substack{a_t\sim\pi(\cdot | s_t) \\ s_{t+1} \sim \prob(\cdot | s_t, a_t)}}\left[\sum_{t=0}^\infty \gamma^t \rew(s_t, a_t) |s_0 = s, a_0 = a\right]
\end{equation*}

We denote by $\prob_{\dpi}$ the action-independent transition function induced by running a policy $\pi$, $\prob_{\dpi}(s'|s) = \sum_{a\in\aspace} \prob(s'|s,a) \pi(a|s)$. Similarly $\rew_{\dpi}(s) = \sum_{a\in\aspace} \rew(s,a) \pi(a|s)$.
We denote $\pistar$ as the optimal policy in $\mdp$; i.e., the policy which maximizes expected future reward. We denote the optimal state and action value functions with respect to $\pistar$ as $\Vstar,\Qstar$.
We denote the stationary distribution of a policy $\pi$ in $\mdp$ by $\xi_\pi$; i.e., 
\begin{equation*}
    \xi_\pi(s) = \sum_{\dot{s}\in\sspace, \dot{a}\in\aspace} \prob(s | \dot{s}, \dot{a}) \pi(\dot{a} | \dot{s}) \xi_\pi(\dot{s}) 
\end{equation*}
We overload notation by also denoting the state-action stationary distribution as $\xipi(s,a) = \xi_\pi(s) \pi(a|s)$. Although only non-terminating MDPs have stationary distributions, a state distribution for terminating MDPs with similar properties exists \citep{Gelada2019OffPolicyDR}.

\subsection{Latent Space Models}
For some MDP $\mdp$, let $\dmdp = \langle\dsspace, \aspace, \drew, \dprob, \gamma\rangle$ be an MDP where $\dsspace$ is a continuous space with metric $\distds$ and a shared action space $\aspace$ between $\mdp$ and $\dmdp$. Furthermore, let $\phi:\sspace\to\dsspace$ be an embedding function which connects the state spaces of these two MDPs. We refer to $(\dmdp, \phi)$ as a \textit{latent space model} of $\mdp$.

Since $\dmdp$ is, by definition, an MDP, value functions can be defined in the standard way. We use $\dVpi,\dQpi$ to denote the value functions of a policy $\dpi \in \dPi$, where $\dPi$ is the set of policies defined on the state space $\dsspace$.The transition and reward functions, $\drew_{\dpi}$ and $\dprob_{\dpi}$, of a policy $\dpi$ are also defined in the standard manner. We use $\dpistar$ to denote the optimal policy in $\dmdp$. The corresponding optimal state and action value functions are then $\dVstar,\dQstar$.
For ease of notation, when $s\in\sspace$, we use $\dpi(\cdot | s) := \dpi(\cdot | \phi(s))$ to denote first using $\phi$ to map $s$ to the state space $\dsspace$ of $\dmdp$ and subsequently using $\dpi$ to generate the probability distribution over actions.

Although similar definitions of latent space models have been previously studied \citep{lavet2018crar, zhang2018solar, ha2018recurrent, oh2017value, hafner2018learning, Kaiser2019ModelBasedRL, Silver2017ThePE}, the parametrizations and training objectives used to learn such models have varied widely. For example \citet{ha2018recurrent, hafner2018learning, Kaiser2019ModelBasedRL} use pixel prediction losses to learn the latent representation while \citep{oh2017value} chooses instead to optimize the model to predict next latent states with the same value function as the sampled next states.

In this work, we study the minimization of loss functions defined with respect to rewards and transitions in the latent space:
\begin{align}
\Ldrew(s,a) &= |\rew(s,a) - \drew(\phi(s),a)| \\
\Ldprob(s,a) &=  \Dist{\phi\prob(\cdot|s,a)}{\dprob(\cdot|\phi(s),a)} \label{eq:latent_transition_loss}
\end{align}
where we use the shorthand notation $\phi\prob(\cdot|s,a)$ to denote the probability distribution over $\dsspace$ of first sampling $s^\prime\sim\prob(\cdot|s,a)$ and then embedding $\ds^\prime = \phi(s^\prime)$, and where $\D$ is a metric between probability distributions.
To provide guarantees, $\mathcal{D}$ in Equation~\ref{eq:latent_transition_loss} needs to be chosen carefully. For the majority of this work, we focus on the Wasserstein metric; in Section~\ref{sec:other_metrics}, we generalize some of the results to alternative metrics from the Maximum Mean Discrepancy family.
\citet{lavet2018crar} and \citet{chung19twotimescale} have considered similar latent losses, but to the best of our knowledge ours is the first theoretical analysis of these models. See Figure \ref{fig:dmdp_models} for an illustration of how the latent space losses are constructed.

We use the term \textit{DeepMDP} to refer to a parameterized latent space model trained via the minimization of losses consisting of $\Ldrew$ and $\Ldprob$ (sometimes referred to as \textit{DeepMDP losses}). In Section~\ref{sec:global}, we derive theoretical guarantees of DeepMDPs when minimizing $\Ldrew$ and $\Ldprob$ over the whole state space. However, our principal objective is to learn DeepMDPs parameterized by deep networks, which requires DeepMDP losses in the form of expectations; we show in Section~\ref{sec:local} that similar theoretical guarantees can be obtained in this setting.

\subsection{Wasserstein Metric}
Initially studied in the optimal transport literature \citep{Villani2008OptimalT}, the Wasserstein-1 (which we simply refer to as the Wasserstein) metric $\wassd{P}{Q}$ between two distributions $P$ and $Q$, defined on a space with metric $d$, corresponds to the minimum cost of transforming $P$ into $Q$, where the cost of moving a particle at point $x$ to point $y$ comes from the underlying metric $d(x,y)$. 
\begin{defn}\label{def:wass}
The Wasserstein-1 metric $\W$ between distributions $P$ and $Q$ on a metric space $\langle \chi, d\rangle$ is:
\begin{equation*}
\wassd{P}{Q} = \inf_{\lambda \in  \Gamma(P,Q)} \int_{\chi\times\chi} d(x,y)\lambda(x,y) \deriv x \deriv y .
\end{equation*}
where $\Gamma(P,Q)$ denotes the set of all couplings of $P$ and $Q$.
\end{defn}
When there is no ambiguity on what the underlying metric $d$ is, we will simply write $\W$. The Monge-Kantorovich duality \citep{Mueller1997IntegralPM} shows that the Wasserstein has a dual form: 
\begin{equation}\label{eq:WassersteinDual}
\wassd{P}{Q} = \sup_{f \in  \F_d} | \expect_{x\sim P} f(x) - \expect_{y \sim Q} f(y) |,
\end{equation}
where $\F_d$ is the set of $1$-Lipschitz functions under the metric $d$, $\F_d = \{ f: | f(x) - f(y) | \le d(x,y) \}$.

\subsection{Lipschitz Norm of Value Functions}
The degree to which a value function of $\dmdp$, $\dVdpi$ approximates the value function $\Vdpi$ of $\mdp$ will depend on the Lipschitz norm of $\dVdpi$. In this section we define and provide conditions for value functions to be Lipschitz.\footnote{Another benefit of MDP smoothness is improved learning dynamics. \citet{pirotta2015policy} suggest that the smaller the Lipschitz constant of an MDP, the faster it is to converge to a near-optimal policy.}
Note that we study the Lipschitz properties of DeepMDPs $\dmdp$ (instead of a MDP $\mdp)$ because in this work, only the Lipschiz properties of DeepMDPs are relevant; the reader should note that these results follow for any continuous MDP with a metric state space.

We say a policy $\dpi\in\dPi$ is \textit{Lipschitz-valued} if its value function is Lipschitz, i.e. it has Lipschitz $\dQdpi$ and $\dVdpi$ functions.
\begin{defn}\label{defn:LipschitzValuedPolicy}
Let $\dmdp$ be a DeepMDP with a metric $\distds$. A policy $\dpi \in \dPi$ is $\KdV$-Lipschitz-valued if for all $\ds_1, \ds_2 \in \dsspace$:
\begin{align*}
    \left | \dVpi(\ds_1) - \dVpi(\ds_2) \right | &\le \KdV \distds(\ds_1, \ds_2),
\end{align*}
and if for all $a \in \aspace$:
\begin{align*}
    \left | \dQpi(\ds_1,a) - \dQpi(\ds_2, a) \right | &\le \KdV \distds(\ds_1, \ds_2).
\end{align*}
\end{defn}

Several works have studied Lipschitz norm constraints on the transition and reward functions \citep{Hinderer2005LipschitzCO,asadi2018lipschitz} to provide conditions for value functions to be Lipschitz. Closely following their formulation, we define Lipschitz DeepMDPs as follows:
\begin{defn}{\label{def:smoothMDP}}
Let $\dmdp$ be a DeepMDP with a metric $\distds$. We say $
\dmdp$ is $(\Kdrew,\Kdprob)$-Lipschitz if, for all $\ds_1, \ds_2 \in \dsspace$ and $a \in \aspace$:
\begin{align*}
    &\left | \drew(\ds_1,a) - \drew(\ds_2, a) \right | \le \Kdrew  \distds(\ds_1, \ds_2) \\
    &\wass{\dprob(\cdot |, \ds_1, a)}{\dprob(\cdot|\ds_2, a)} \le \Kdprob  \dists(\ds_1, \ds_2)
\end{align*}
\end{defn}
From here onwards, we will we restrict our attention to the set of Lipschitz DeepMDPs for which the constant $\Kdprob$ is sufficiently small, formalized in the following assumption:
\begin{assum}\label{assum:lipschitz_assumption}
 The Lipschitz constant $\Kdprob$ of the transition function $\dprob$ is strictly smaller than $\tfrac{1}{\gamma}$.
\end{assum}
From a practical standpoint, Assumption \ref{assum:lipschitz_assumption} is relatively strong, but simplifies our analysis by ensuring that close states cannot have future trajectories that are ``divergent.'' An MDP might still not exhibit divergent behaviour even when $\Kdprob\ge\frac{1}{\gamma}$. In particular, when episodes terminate after a finite amount of time, Assumption \ref{assum:lipschitz_assumption} becomes unnecessary. We leave as future work how to improve on this assumption.

We describe a small set of Lipschitz-valued policies. For any policy $\dpi\in\dPi$, we refer to the Lipschitz norm of its transition function $\dprob_{\dpi}$ as ${\Kprob}_{\dpi} \ge \wass{{\dprob}_{\dpi}(\cdot|\ds_{1})}{{\dprob}_{\dpi}(\cdot|\ds_{2})}$ for all $\ds_1, \ds_2\in\sspace$. Similarly, we denote the Lipschitz norm of the reward function as ${\Kdrew}_{\dpi} \ge | {\drew}_{\dpi}(\ds_1) - {\drew}_{\dpi}(\ds_2) |$.

\begin{restatable}{lemma}{LipschitzValue}\label{lem:LipschitzValue}
Let $\dmdp$ be $({\Kdrew},{\Kdprob})$-Lipschitz. Then,
\begin{enumerate}
\item The optimal policy $\dpistar$ is $\frac{\Kdrew}{1-\gamma{\Kdprob}}$-Lipschitz-valued.
\item All policies with ${\Kdprob}_{\dpi} \le \frac{1}{\gamma}$  are $\frac{{\Kdrew}_{\dpi}}{1-\gamma K_{{\dprob}_{\dpi}}}$-Lipschitz-valued.
\item All constant policies (i.e. ${\dpi}(a|{\ds}_1) = {\dpi}(a|{\ds}_2), \forall a\in\aspace, \ds_1, {\ds}_2 \in\dsspace$) are  $\frac{\Kdrew}{1-\gamma{\Kdprob}}$-Lipschitz-valued.
\end{enumerate}
\end{restatable}
\begin{proof}
See Appendix~\ref{sec:Proofs} for all proofs.
\end{proof}
A more general framework for understanding Lipschitz value functions is still lacking.
Little prior work studying classes of Lipschitz-valued policies exists in the literature and we believe that this is an important direction for future research.
\section{Global DeepMDP Bounds}\label{sec:global}
We now present our first main contributions: concrete DeepMDP losses, and several bounds which provide us with useful guarantees when these losses are minimized.
We refer to these losses as the \textit{global} DeepMDP losses, to emphasize their dependence on the whole state and action space:\footnote{The $\infty$ notation is a reference to the $\ell_\infty$ norm}
\begin{align}
&\gLdrew = \sup_{s\in\sspace,a\in\aspace} |\rew(s,a) - \drew(\phi(s),a)| \\
&\gLdprob = \sup_{s\in\sspace,a\in\aspace} \wass{\phi\prob(\cdot|s,a)}{\dprob(\cdot|\phi(s),a)}
\end{align}

\subsection{Value Difference Bound}
We start by bounding the difference of the value functions $\Qdpi$ and $\dQdpi$ for any policy $\dpi\in\dPi$. Note that $\Qdpi(s,a)$ is computed using $\prob$ and $\rew$ on $\sspace$ while $\dQdpi(\phi(s), a)$ is computed using $\dprob$ and $\drew$ on $\dsspace$.
\begin{restatable}{lemma}{globalValueDifferenceBound}\label{lem:globalValueDifferenceBound}
Let $\mdp$ and $\dmdp$ be an MDP and DeepMDP respectively, with an embedding function $\phi$ and global loss functions $\gLdrew$ and $\gLdprob$. For any $\KdV$-Lipschitz-valued policy $\dpi \in \dPi$ the value difference can be bounded by
\begin{equation*}
\left | \Qdpi (s, a) - \dQdpi (\phi(s) ,a) \right |  \leq 
\frac{  \gLdrew  + \gamma \KdV \gLdprob  }{1-\gamma} ,
\end{equation*}
\end{restatable}

The previous result holds for all policies $\dPi \subseteq \Pi$, a subset of all possible policies $\Pi$. The reader might ask whether this is an interesting set of policies to consider; in Section \ref{sec:bisimulation}, we answer with a fat ``yes'' by characterizing this set via a connection with bisimulation.

A bound similar to Lemma~\ref{lem:globalValueDifferenceBound} can be found in \citet{asadi2018lipschitz}, who study non-latent transition models using the Wasserstein metric when there is access to an exact reward function. We also note that our results are arguably simpler, since we do not require the treatment of MDP transitions in terms of distributions over a set of deterministic components.

\subsection{Representation Quality Bound}\label{sec:globalRepresentationQuality}
When a representation is used to predict the value of a policy in $\mdp$, a clear failure case is when two states with different values are collapsed to the same representation.
The following result demonstrates that when the global DeepMDP losses $\gLdrew=0$ and $\gLdprob=0$, this failure case can never occur for the embedding function $\phi$.
\begin{restatable}{theorem}{globalLipschitzRepresentationQuality}\label{thm:globalLipschitzRepresentationQuality}
\label{thm:subopt}
Let $\mdp$ and $\dmdp$ be an MDP and DeepMDP respectively, let $\distds$ be a metric in $\dsspace$, $\phi$ be an embedding function and $\gLdrew$ and $\gLdprob$ be the global loss functions. For any $\KdV$-Lipschitz-valued policy $\dpi \in \dPi$ the representation $\phi$ guarantees that for all $s_1, s_2 \in \sspace$ and $a\in\aspace$,
\begin{multline*}
\left | \Qdpi (s_1, a) - \Qdpi (s_2 ,a) \right |  \leq  \KdV \distds(\phi(s_1), \phi(s_2)) \\
+ 2\frac{ \gLdrew  + \gamma \KdV \gLdprob  }{1-\gamma}
\end{multline*}
\end{restatable}
This result justifies learning a DeepMDP and using the embedding function $\phi$ as a representation to predict values. A similar connection between the quality of representations and model based objectives in the linear setting was made by \citet{Parr2008AnAO}.

\subsection{Suboptimality Bound}
For completeness, we also bound the performance loss of running the optimal policy of $\dmdp$ in $\mdp$, compared to the optimal policy $\pistar$. See Theorem \ref{thm:globalSubOpt} in Appendix \ref{sec:Proofs}.
\section{Local DeepMDP Bounds}
\label{sec:local}

In large-scale tasks, data from many regions of the state space is often unavailable,\footnote{Challenging exploration environments like Montezuma's Revenge are a prime example.} making it infeasible to measure -- let alone optimize -- the global losses. Further, when the capacity of a model is limited, or when sample efficiency is a concern, it might not even be desirable to precisely learn a model of the whole state space.
Interestingly, we can still provide similar guarantees based on the DeepMDP losses, as measured under an \textit{expectation} over a state-action distribution, denoted here as $\xi$. We refer to these as the losses \textit{local} to $\xi$.
Taking $\lLdrewxi$, $\lLdprobxi$ to be the reward and transition losses under $\xi$, respectively, we have the following local DeepMDP losses:
\begin{align}
&\lLdrewxi = \expxi | \rew(s,a) - \drew(\phi(s), a) |, \label{eqn:LocalRLoss} \\
&\lLdprobxi = \expxi \left[ \wass{ \phi \prob(\cdot | s, a)}{ \dprob(\cdot | \phi(s), a) } \right].  \label{eqn:LocalPLoss}
\end{align}
Losses of this form are compatible with the stochastic gradient decent methods used by neural networks. Thus, study of the local losses allows us to bridge the gap between theory and practice.

\subsection{Value Difference Bound}
\label{sec:local-value-function}
We provide a value function bound for the local case, analogous to Lemma~\ref{lem:globalValueDifferenceBound}.
\begin{restatable}{lemma}{localValueDiffBoundW}\label{lem:localValueDiffBoundW}
Let $\mdp$ and $\dmdp$ be an MDP and DeepMDP respectively, with an embedding function $\phi$. For any $\KdV$-Lipschitz-valued policy $\dpi \in \dPi$, the expected value function difference can be bounded using the local loss functions $\lLdrew$ and $\lLdprob$ measured under $\xidpi$, the stationary state action distribution of $\dpi$.
\begin{multline*}
\expxidpi \left | \Qdpi (s, a) - \dQdpi (\phi(s) ,a) \right |   \leq 
\frac{ \lLdrew  + \gamma \KdV \lLdprob }{1-\gamma} ,
\end{multline*}
\end{restatable}

The provided bound guarantees that for any policy $\dpi \in \dPi$ which visits state-action pairs $(s,a)$ where $\Ldrew(s,a)$ and $\Ldprob(s,a)$ are small, the DeepMDP will provide accurate value functions for any states likely to be seen under the policy.\footnote{The value functions might be inaccurate in states that the policy $\dpi$ rarely visits.}

\subsection{Representation Quality Bound}
\label{sec:local-representation-similarity}
We can also extend the local value difference bound to provide a local bound on how well the representation $\phi$ can be used to predict the value function of a policy $\bar \pi \in \bar \Pi$, analogous to Theorem~\ref{thm:globalLipschitzRepresentationQuality}.
\begin{restatable}{theorem}{localLipschitzRepresentationQuality}
\label{thm:localLipschitzRepresentationQuality}
Let $\mdp$ and $\dmdp$ be an MDP and DeepMDP respectively, let $\distds$ be the metric in $\dsspace$ and $\phi$ be the embedding function. Let  $\dpi \in \dPi$ be any $\KdV$-Lipschitz-valued policy with stationary distribution $\xi_{\dpi}$, and let $\lLdrew$ and $\lLdprob$ be the local loss functions. For any two states $s_1, s_2 \in \sspace$, the representation $\phi$ is such that,
\begin{multline*}
| \Vdpi(s_1) - \Vdpi(s_2) | \leq \KdV \distds(\phi(s_1), \phi(s_2)) \\
+ \frac{
 \lLdrew  + \gamma \KdV \lLdprob
}{1-\gamma}
\left ( \frac{1}{d_{\dpi}(s_1)} + \frac{1}{d_{\dpi}(s_2)}  \right )
\end{multline*}
\end{restatable}
\vspace{-5px}
Thus, the representation quality argument given in \ref{sec:globalRepresentationQuality} holds for any two states $s_1$ and $s_2$ which are visited often by a policy $\dpi$.

\section{Bisimulation}\label{sec:bisimulation}

\begin{figure}
    \centering
    \includegraphics[keepaspectratio, width=.35\textwidth]{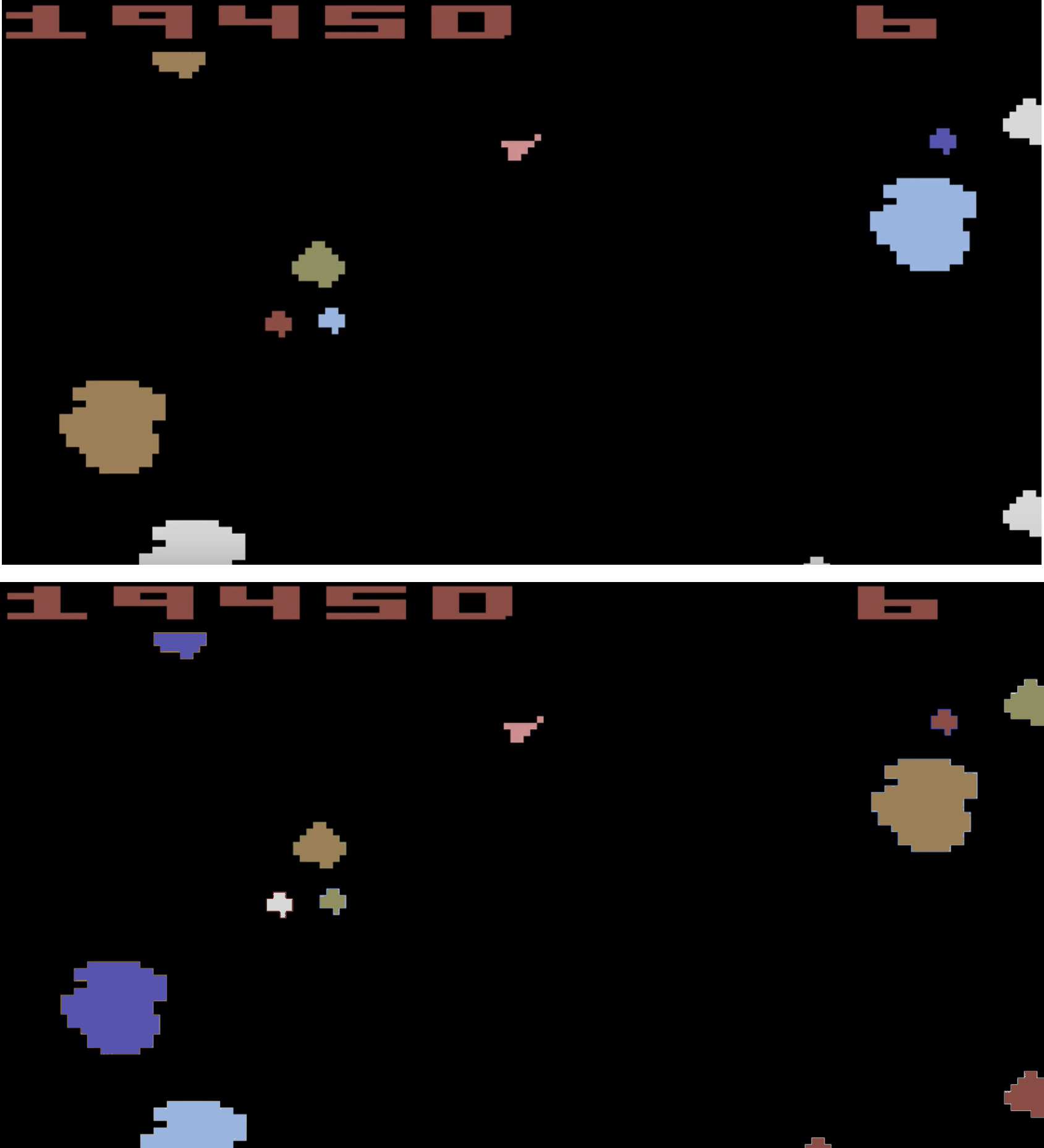}
    \caption{A pair of bisimilar states. In the game of \textsc{Asteroids}, the colors of the asteroids can vary randomly, but this in no way impacts gameplay.}
    \vspace{-15px}
    \label{fig:bisimulation}
\end{figure}

\subsection{Bisimulation Relations}
Bisimulation relations in the context of RL \citep{givan2003equivalence}, are a formalization of behavioural equivalence between states. 

\begin{defn}[\citet{givan2003equivalence}]
Given an MDP $\mdp$, an equivalence relation $B$ between states is a bisimulation relation if for all states $s_1, s_2 \in \sspace$ that are equivalent under $B$ (i.e. $s_1 B s_2$), the following conditions hold for all actions $a\in\aspace$.
\begin{align*}
    &R(s_1, a) = R(s_2, a) \\
    &\prob(G | s_1, a) = \prob(G | s_2, a),   \forall G \in \sspace / B
\end{align*}
Where $\sspace/B$ denotes the partition of $\sspace$ under the relation $B$, the set of all groups of equivalent states, and where $\prob(G |s,a) = \sum_{s' \in G} \prob(s' |s,a)$.
\end{defn}
Note that bisimulation relations are not unique. For example, the equality relation $=$ is always a bisimulation relation.
Of particular interest is the maximal bisimulation relation $\sim$, which defines the partition $\sspace / \sim$ with the fewest elements (or equivalently, the relation that generates the largest possible groups of states). We will say that two states are bisimilar if they are equivalent under $\sim$.
Essentially, two states are bisimilar if (1) they have the same immediate reward for all actions and
(2) both of their distributions over next-states contain states which themselves are bisimilar. Figure~\ref{fig:bisimulation} gives an example of states that are bisimilar in the Atari 2600 game \textsc{Asteroids}.
An important property of bisimulation relations is that any two bisimilar states $s_1, s_2$ must have the same optimal value function $\Qstar(s_1,a) = \Qstar(s_2, a), \forall a\in\aspace$. Bisimulation relations were first introduced for state aggregation \citep{givan2003equivalence}, which is a form of representation learning, since merging behaviourally equivalent states does not result in the loss of information necessary for solving the MDP.

\subsection{Bisimulation Metrics}
A drawback of bisimulation relations is their \textit{all-or-nothing} nature. Two states that are nearly identical, but differ slightly in their reward or transition functions, are treated as though they were just as unrelated as two states with nothing in common. Relying on the optimal transport perspective of the Wasserstein, \citet{ferns2004bisimulation} introduced bisimulation metrics, which are pseudometrics that quantify the behavioural similarity of two discrete states.

A pseudometric $d$ satisfies all the properties of a metric except \textit{identity of indiscernibles}, $d(x, y) = 0 \Leftrightarrow x=y$. A pseudometric can be used to define an equivalence relation by saying that two points are equivalent if they have zero distance; this is called the
kernel of the pseudometric. Note that pseudometrics must obey the triangle inequality, which ensures the kernel satisfies the associative property.
Without any changes to its definition, the Wasserstein metric can be extended to spaces $\langle \chi, d \rangle$, where $d$ is a pseudometric.
Intuitively, the usage of a pseudometric in the Wasserstein can be interpreted as allowing different points $x_1 \neq x_2$ in $\chi$ to be equivalent under the pseudometric (i.e. $d(x_1, x_2) = 0$). Thus, there is no need for transportation from one to the other.

An extension of bisimulation metrics based on Banach fixed points by \citet{ferns2011bisimulation} which allows the metric to be defined for MDPs with discrete and continuous state spaces. 

\begin{defn}[\citet{ferns2011bisimulation}]
\label{def:bisimulation_metric}
Let $\mdp$ be an MDP and denote by $Z$ the space of pseudometrics on the space $\sspace$ s.t. $d(s_1, s_2) \in [0, \infty)$ for $d\in Z$. Define the operator $F:Z \to Z$ to be:

\begin{multline*}
    F_d(s_1, s_2) = \max_a (1 - \gamma) \left| \rew(s_1, a) - \rew(s_2,a) \right | \\ + \gamma \W_{d} (\prob(\cdot | s_1, a), \prob(\cdot | s_2, a)) .
\end{multline*}
Then:
\begin{enumerate}
  \item The operator $F$ is a contraction with a unique fixed point denoted by $\td$.
   \item The kernel of $\td$ is the maximal bisimulation relation $\sim$.  (i.e. $\td(s_1, s_2) = 0 \iff s_1 \sim s_2$)
\end{enumerate}
\end{defn}

A useful property of bisimulation metrics is that the optimal value function difference between any two states can be upper bounded by the bisimulation metric between the two states.
\begin{equation*}
|\Vstar(s_1) - \Vstar(s_2)| \le \frac{\td(s_1, s_2)}{1-\gamma}
\end{equation*}

Bisimulation metrics have been used for state aggregation \citep{ferns2004bisimulation, Ruan2015RepresentationDF}, feature discovery \citep{Comanici2011BasisFD} and transfer learning between MDPs
\citep{castro2010bisimulation}, but due to their high computational cost and poor compatibility with deep networks they have not been successfully applied to large scale settings.

\subsection{Connection with DeepMDPs}
The representation $\phi$ learned by global DeepMDP losses with the Wasserstein metric can be connected to bisimulation metrics.
\begin{restatable}{theorem}{BisimulationMetricUpperBoundPhi}\label{thm:BisimulationMetricUpperBoundPhi}
Let $\mdp$ be an MDP and $\dmdp$ be a $\Kdrew$-$\Kdprob$-Lipschitz DeepMDP with metric $\distds$. Let $\phi$ be the embedding function and $\gLdprob$ and $\gLdrew$ be the global DeepMDP losses. The bisimulation distance in $\mdp$, $\td: \sspace \times \sspace \to \mathbb{R}^+$ can be upperbounded by the $\ell_2$ distance in the embedding and the losses in the following way:
\begin{multline*}
\td(s_1, s_2) \le \frac{(1-\gamma)K_{\rew}}{1 - \gamma \Kdprob} \distds(\phi(s_1), \phi(s_2)) \\ 
+ 2 \left(  \gLdrew + \gamma \gLdprob \frac{\Kdrew}{1 - \gamma \Kdprob} \right)     
\end{multline*}
\end{restatable}
This result provides a similar bound to Theorem~\ref{thm:globalLipschitzRepresentationQuality}, except that instead of bounding the value difference $|\dVdpi(s_1) - \dVdpi(s_2)|$ the bisimulation distance $\td(s_1,s_2)$ is bounded. We speculate that similar results should be possible based on \textit{local} DeepMDP losses, but they would require a generalization of bisimulation metrics to the local setting.

\subsection{Characterizing $\dPi$}\label{sec:characterizing_dPi}
In order to better understand the set of policies $\dPi$ (which appears in the bounds of Sections~\ref{sec:global}~and~\ref{sec:local}), we first consider the set of \textit{bisimilar policies}, defined as $\tPi = \left \{\pi: \forall s_1, s_2\in\sspace, s_1\sim s_2 \Leftrightarrow \pi(a|s_1) = \pi(a|s_2) \forall a  \right \}$, which contains all policies that act the same way on states that are bisimilar. Although this set excludes many policies in $\Pi$, we argue that it is adequately expressive, since any policy that acts differently on states that are bisimilar is fundamentally uninteresting.\footnote{For control, searching over these policies increases the size of the search space with no benefits on the optimality of the solution.}

We show a connection between deep policies and bisimilar policies by proving that the set of \textit{Lipschitz-deep policies}, $\dPi_K \subset \dPi$, approximately contains the set of \textit{Lipschitz-bisimilar policies}, $\tPi_K \subset \tPi$, defined as follows:
\begin{align*}
&\dPi_K = \left\{\dpi: \forall s_1\neq s_2 \in\sspace,  \frac{|\dpi(a|s_1) - \dpi(a|s_2)|}{\distds(\phi(s_1), \phi(s_2))} \le K  \right \}, \\
&\tPi_K = \left \{\pi: \forall s_1\neq s_2 \in\sspace,  \frac{|\pi(a|s_1) - \pi(a|s_2)|}{\td(s_1, s_2)} \le K \right \}.
\end{align*}

The following theorem proves that minimizing the global DeepMDP losses ensures that for any $\tpi\in\tPi_K$, there is a deep policy $\dpi\in\dPi_{CK}$ which is close to $\tpi$, where the constant $C=\frac{(1-\gamma)K_{\rew}}{1 - \gamma \Kdprob}$.
\begin{restatable}{theorem}{dPiBisim}\label{thm:dPiBisim}
Let $\mdp$ be an MDP and $\dmdp$ be a ($\Kdrew$, $\Kdprob$)-Lipschitz DeepMDP, with an embedding function $\phi$ and global loss functions $\gLdrew$ and $\gLdprob$. Denote by $\tPi_K$ and $\dPi_K$ the sets of Lipschitz-bisimilar and Lipschitz-deep policies.
Then for any $\tpi\in\tPi_K$ there exists a $\dpi\in\dPi_{CK}$ which is close to $\tpi$ in the sense that, for all $s \in \sspace$ and $a \in \aspace$,
\begin{equation*}
| \tpi(a |s) - \dpi(a|s) | \le  \gLdrew + \gamma \gLdprob \frac{\Kdrew}{1 - \gamma \Kdprob}
\end{equation*}
\end{restatable}


\begin{figure*}[t]
    \centering
    \includegraphics[keepaspectratio, width=.8\textwidth]{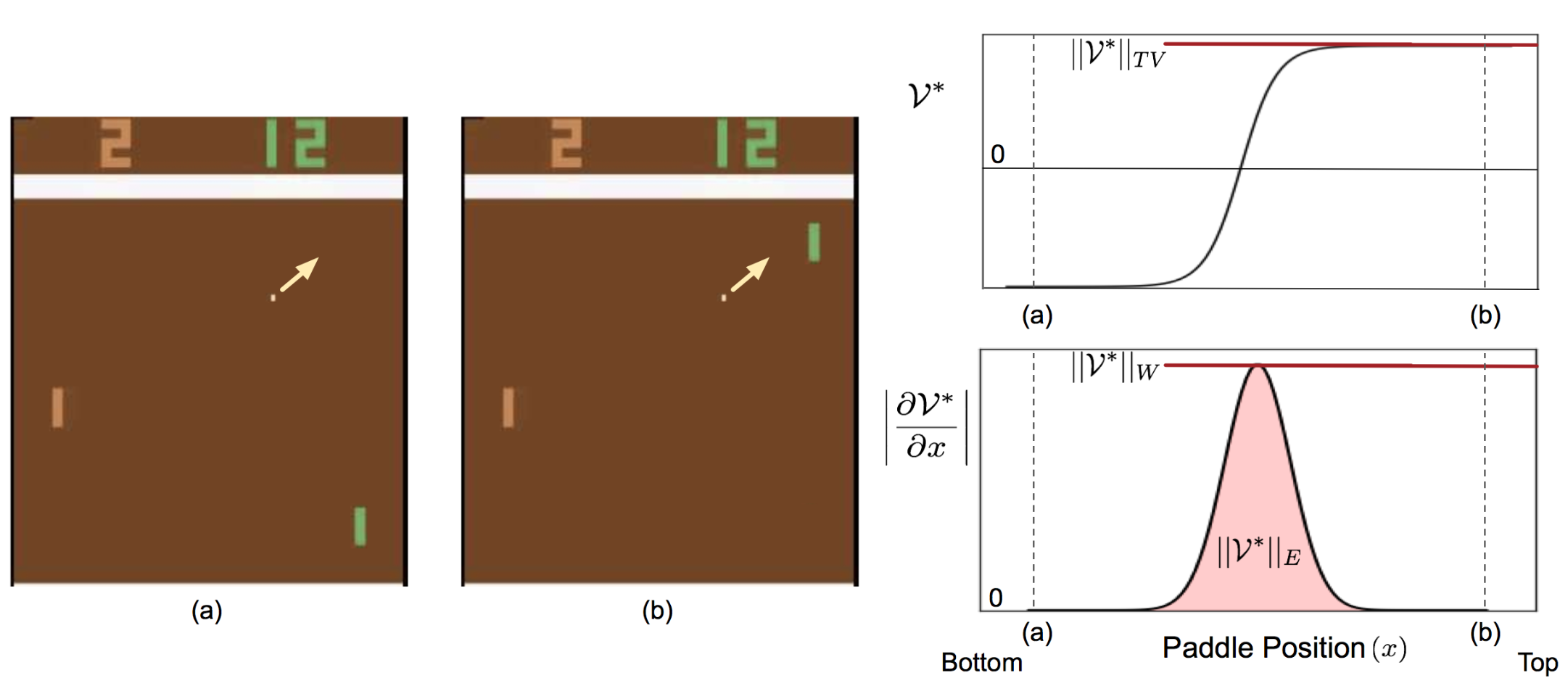}
    \caption{Visualization of the way in which different smoothness properties on the value function are derived.  The left compares two near-identical frames of \textsc{Pong}, (a) and (b), whose only difference is the position of the player's paddle. The plots on the right show the optimal value of the state (top) and the derivative of the optimal value (bottom) as a function of the position of the player's paddle, assuming all other features of the state are kept constant. The associated smoothness of each Norm-MMD metric is shown visually. (Note that this is for illustrative purposes only, and was not actually computed from the real game. The curve in the value function represents noisy dynamics, such as those induced by ``sticky actions'' \citep{Mnih2015HumanlevelCT}; if the environment were deterministic, the optimal value would be a step function.)}
    \label{fig:dmdp_lipschitz_pong}
\end{figure*}

\section{Beyond the Wasserstein}
\label{sec:other_metrics}
Interestingly, value difference bounds (Lemmas~\ref{lem:globalValueDifferenceBound}~and~\ref{lem:localValueDiffBoundW}) can be derived for many different choices of probability metric $\D$ (in the DeepMDP transition loss function, Equation~\ref{eq:latent_transition_loss}).
Here, we generalize the result to a family of \textit{Maximum Mean Discrepancy (MMD)} metrics \citep{Gretton2012AKT} defined via a function norm that we denote as \textit{Norm Maximum Mean Discrepancy (Norm-MMD)} metrics. Interestingly, the role of the Lipschitz norm in the value difference bounds is a consequence of using the Wasserstein; when we switch from the Wasserstein to another metric, it is replaced by a different term. We interpret these terms as different forms of smoothness of the value functions in $\dmdp$.

By choosing a metric whose associated smoothness corresponds well to the environment, we can potentially improve the tightness of the bounds. 
For example, in environments with highly non-Lipschitz dynamics, it may be impossible to learn an accurate DeepMDP whose deep value function has a small Lipschitz norm. Instead, the associated smoothness of another metric might be more appropriate. Another reason to consider other metrics is computational; the Wasserstein has high computational cost and suffers from biased stochastic gradient estimates \citep{bińkowski2018demystifying, bellemare2017cramer}, so minimizing a simpler metric, such as the KL, may be more convenient.

\subsection{Norm Maximum Mean Discrepancy Metrics}
MMD metrics \citep{Gretton2012AKT} are a family of probability metrics, each generated via a class of functions.
They have also been studied by \citet{muller1997integral} under the name of Integral Probability Metrics.
\begin{defn}[\citet{Gretton2012AKT} Definition 2]\label{def:MMD}
Let $P$ and $Q$ be distributions on a measurable space $\chi$ and let $\F_\D$ be a class of functions $f: \chi\to\mathbb{R}$. The Maximum Mean Discrepancy $\D$ is
\begin{equation*}
\Dist{P}{Q} = \sup_{f \in  \F_\D} | \expect_{x\sim P} f(x) - \expect_{y \sim Q} f(y) |.
\end{equation*}
\end{defn}
When $P = Q$ it's obvious that $\Dist{P}{Q} = 0$ regardless of the function class $\F_\D$. But the class of functions leads to MMD metrics with different behaviours and properties. Of interest to us are function classes generated via function seminorms\footnote{A seminorm $\normD{\cdot}$ is a norm except that $\normD{f} =0 \not\Rightarrow f = 0$.}. Concretely, we define a Norm-MMD metric $\D$ to be an MMD metric generated from a function class $\F_D$ of the following form:
\begin{equation*}
\F_\D = \{ f: \normD{f} \le 1 \}.
\end{equation*}
where $\normD{\cdot}$ is the \textit{associated function seminorm} of $\D$. We will see that the family of Norm-MMDs are well suited for the task of latent space modeling. 
Their key property is the following: let $\D$ be a Norm-MMD, then for any function $f$ s.t. $\normD{f} \le K$,
\begin{equation}
 | \expect_{x\sim P} f(x) - \expect_{y \sim Q} f(y) | \le K\cdot \Dist{P}{Q}.
\end{equation}

We now discuss three particularly interesting examples of Norm-MMD metrics.

\textbf{Total Variation:} Defined as $TV(P, Q) = \frac{1}{2} \int_\chi |P(x) - Q(x)| \deriv x$, the Total Variation is one of the most widely-studied metrics. Pinsker's inequality \citep[p.63]{convanal} bounds the TV with the Kullback--Leibler (KL) divergence. 
The Total Variation is also the Norm-MMD generated from the set of functions with absolute value bounded by $1$ \citep{muller1997integral}. Thus, the function norm $\norm{f}_{TV} = \norm{f}_\infty  = \sup_{x\in\chi} | f(x) |$.

\textbf{Wasserstein metric:}
The interpretation of the Wasserstein as an MMD metrics is clear from its dual form (Equation \ref{eq:WassersteinDual}), where the function class $\F_\W$ is set of $1$-Lipschitz functions,
\begin{equation*}
\F_\W = \{ f: |f(x) - f(y)| \le d(x,y), \forall x,y \in \chi \}.
\end{equation*}
The norm associated with the Wasserstein metric $\norm{f}_\W$ is therefore the Lipschitz norm, which in turn is the the $\ell_\infty$ norm of $f'$ (the derivative of $f$). Thus, $\norm{f}_\W =  \norm{f'}_\infty = \sup_{x\in\chi} \frac{\deriv f(x)}{\deriv x}$.

\textbf{Energy distance:} The energy distance $E$ was first developed to compare distributions in high dimensions via a two sample test \citep{Szkely2004TESTINGFE, Gretton2012AKT}. It is defined as:
\begin{multline*}
E(P, Q) = 2 \expect_{(x,y)\sim P\times Q} \norm{x - y} \\ - \expect_{x,x' \sim P} \norm{x - x'} - \expect_{y,y'\sim Q} \norm{y - y'},
\end{multline*}
where $x,x'\sim P$ denotes two independent samples of the distribution $P$. \citet{Sejdinovic2013EquivalenceOD} showed the connection between the energy distance and MMD metrics. Similarly to the Wasserstein, the Energy distance's associated seminorm is: $\norm{f}_E =  \norm{f'}_1 = \int_\chi |\frac{\deriv f(x)}{\deriv x} | \deriv x$.


\subsection{Value Function Smoothness}

In the context of value functions, we interpret the function seminorms associated with Norm-MMD metrics as different forms of \textit{smoothness}. 
\begin{defn}\label{defn:SmoothValue}
Let $\dmdp$ be a DeepMDP and let $\D$ be a Norm-MMD with associated norm $\normD{\cdot}$. We say that a policy $\dpi \in \dPi$ is $\KdV$-smooth-valued if: 
\begin{equation*}
    \normD{\dVdpi} \le \KdV.
\end{equation*}
and if for all $a \in \aspace$:
\begin{equation*}
    \normD{\dQdpi(\cdot, a)} \le \KdV.
\end{equation*}
\end{defn}

For a value function $\dVdpi$, $\norm{\dVdpi}_{TV}$ is the maximum absolute value of $\dVdpi$.
Both $\norm{\dVdpi}_{\W}$ and $\norm{\dVdpi}_{E}$ depend on the derivative of $\dVdpi$, but while $\norm{\dVdpi}_\W$ is governed by point of maximal change, $\norm{\dVdpi}_E$ instead measures the amount of change over the whole state space $\ds$. Thus, a value function with a small region of high derivative (and thus, large $\norm{\dVdpi}_\W$) can still have small $\norm{\dVdpi}_E$. 
In Figure~\ref{fig:dmdp_lipschitz_pong} we provide an intuitive visualization of these three forms of smoothness in the game of Pong.

One advantage of the Total Variation is that it requires minimal assumptions on the DeepMDP. If the reward function is bounded, i.e. $ |\drew(\ds,a)| \le \Kdrew,~ \forall \ds\in\dsspace, a\in\aspace$, then all policies $\dpi\in\dPi$ are $\frac{\Kdrew}{1-\gamma}$-smooth-valued. We leave it to future work to study value function smoothness more generally for different Norm-MMD metrics and their associated norms.

\subsection{Generalized Value Difference Bounds}
The global and local value difference results
(Lemmas~\ref{lem:globalValueDifferenceBound}~and~\ref{lem:localValueDiffBoundW}), as well as the suboptimality result Lemma~\ref{thm:subopt}, can easily be derived when $\D$ is any Norm-MMD metric. Due to the repetitiveness of these results, we don't include them in the main paper; refer to Appendix~\ref{sec:generalizedProofs} for the full statements and proofs.
We leave it to future work to characterize the of policies $\dPi$ when general (i.e. non-Wasserstein) Norm-MMD metrics are used.

The fact that the representation quality results (Theorems~\ref{thm:globalLipschitzRepresentationQuality}~and~\ref{thm:localLipschitzRepresentationQuality}) and the connection with bisimulation (Theorems~\ref{thm:BisimulationMetricUpperBoundPhi}~and~\ref{thm:dPiBisim}) don't generalize to Norm-MMD metrics emphasizes the special role the Wasserstein metric plays for representation learning. 

\section{Related Work in Representation Learning}
State aggregation methods \cite{abel2017near,li2006towards, singh1995reinforcement, givan2003equivalence, jiang2015abstraction, Ruan2015RepresentationDF} attempt to reduce the dimensionality of the state space by joining states together, taking the perspective that a good representation is one that reduces the total number of states without sacrificing any necessary information.
Other representation learning approaches take the perspective that an optimal representation contains features that allow for the linear parametrization of the optimal value function \citep{Comanici2011BasisFD, Mahadevan2007ProtovalueFA}. Recently, \citet{Bellemare2019AGP, Dadashi2019TheVF} approached the representation learning problem from the perspective that a good representation is one that allows the prediction via a linear map of any value function in the value function space. 
In contrast, we have argued that a good representation (1) allows for the parametrization of a large set of \emph{interesting} policies and (2) allows for the good approximation of the \emph{value function} of these policies.

Concurrently, a suite of methods combining model-free deep reinforcement learning with auxiliary tasks has shown large benefits on a wide variety of domains
\citep{jaderberg2016reinforcement, Oord2018RepresentationLW, Mirowski2017LearningTN}. Distributional RL \citep{bellemare17distributional}, which was not initially introduced as a representation learning technique, has been shown by \citet{Lyle2019ACA} to only play an auxiliary task role. Similarly, \cite{Fedus2019HyperbolicDA} studied different discounting techniques by learning the spectrum of value functions for different discount values $\gamma$, and incidentally found that to be a highly useful auxiliary task. Although successful in practice, these auxiliary task methods currently lack strong theoretical justification. Our approach also proposes to minimize losses as an auxilliary task for representation learning, for a specifc choice of losses: the DeepMDP losses. We have formally justified this choice of losses, by providing theoretical guarantees on representation quality.

\begin{figure}[t]
    \centering
    \subfigure[One-track DonutWorld.] {
    \makebox[0.45\textwidth]{
      \includegraphics[keepaspectratio, width=0.35\textwidth]{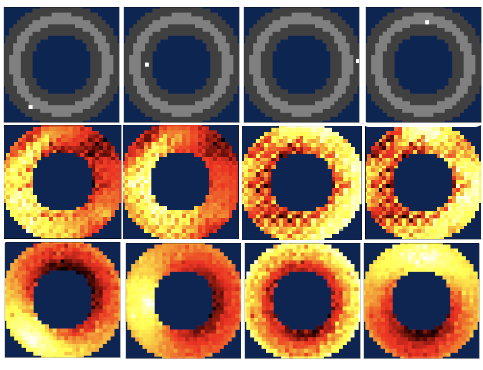}
      }
      \label{fig:toyheatmaps}
    }
    \vspace{0.3in}
    \subfigure[Four-track DonutWorld.] {
    \makebox[0.45\textwidth]{
      \includegraphics[keepaspectratio, width=0.35\textwidth]{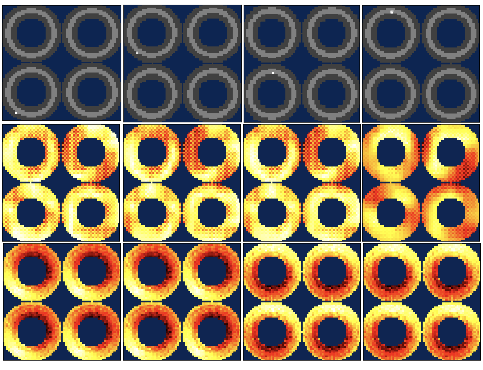}
      }
      \label{fig:multitoyheatmaps}
    }
    \vspace{-0.4in}
    \caption{Given a state in our DonutWorld environment (first row), we plot a heatmap of the distance between that latent state and each other latent state, for both autoencoder representations (second row) and DeepMDP representations (third row). More-similar latent states are represented by lighter colors.}
    \label{fig:toy}
    \vspace{-0.3in}
\end{figure}

\section{Empirical Evaluation}\label{sec:empirical}
Our results depend on minimizing losses in expectation, which is the main requirement for deep networks to be applicable. Still, two main obstacles arise when turning these theoretical results into practical algorithms:

\textbf{(1) Minimization of the Wasserstein} \citet{Arjovsky2017WassersteinGA} first proposed the use of the Wasserstein distance for Generative Adversarial Networks (GANs) via its dual formulation (see Equation~\ref{eq:WassersteinDual}). Their approach consists of training a network, constrained to be $1$-Lipschitz, to attain the supremum of the dual. Once this supremum is attained, the Wasserstein can be minimized by differentiating through the network.
Quantile regression has been proposed as an alternative solution to the minimization of the Wasserstein  \citep{Dabney2018DistributionalRL}, \citep{Dabney2018ImplicitQN}, and has shown to perform well for Distributional RL. The reader might note that issues with the stochastic minimization of the Wasserstein distance have been found to be biased by \citet{bellemare2017cramer} and \citet{bińkowski2018demystifying}. 
In our experiments, we circumvent these issues by assuming that both $\prob$ and $\dprob$ are deterministic. This reduces the Wasserstein distance $\W_{\distds} (\phi\prob(\cdot|s,a), \dprob(\cdot|\phi(s),a))$ to $\distds(\phi(\prob(s,a)), \dprob(\phi(s),a))$, where $\prob(s,a)$ and $\dprob(\ds,a)$ denote the deterministic transition functions.

\textbf{(2) Control the Lipschitz constants $\Kdrew$ and $\Kdprob$}. We also turn to the field of Wasserstein GANs for approaches to constrain deep networks to be Lipschitz. Originally, \citet{Arjovsky2017WassersteinGA} used a projection step to constraint the discriminator function to be $1$-Lipschitz. \citet{Gulrajani2017ImprovedTO} proposed using a gradient penalty, and sowed improved learning dynamics. Lipschitz continuity has also been proposed as a regularization method by \citet{Gouk2018RegularisationON}, who provided an approach to compute an upper bound to the Lipschitz constant of neural nets. In our experiments, we follow \citet{Gulrajani2017ImprovedTO} and utilize the gradient penalty.

\begin{figure}
    \centering
    \includegraphics[keepaspectratio, width=.4\textwidth]{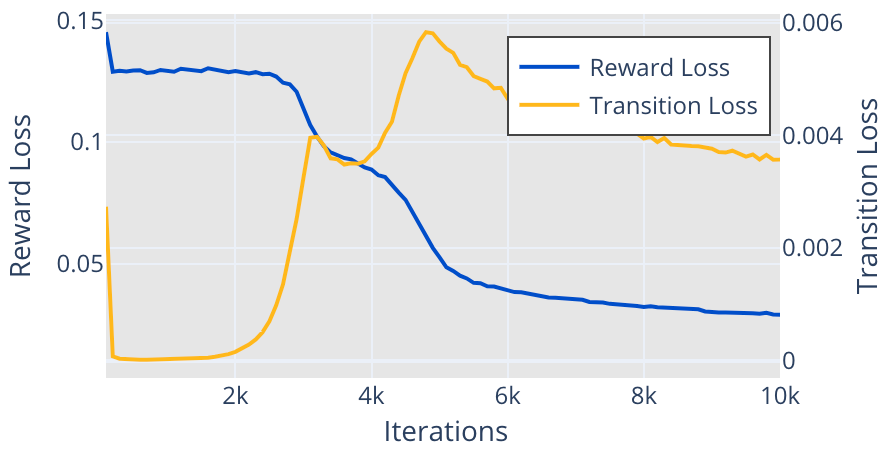}
    \caption{Due to the competition between reward and transition losses, the optimization procedure spends significant time in local minima early on in training. It eventually learns a good representation, which it then optimizes further. (Note that the curves use different scaling on the y-axis.)}
    \label{fig:competeloss}
\end{figure}

\begin{figure*}[t]
    \centering
    \includegraphics[keepaspectratio, width=0.9\textwidth]{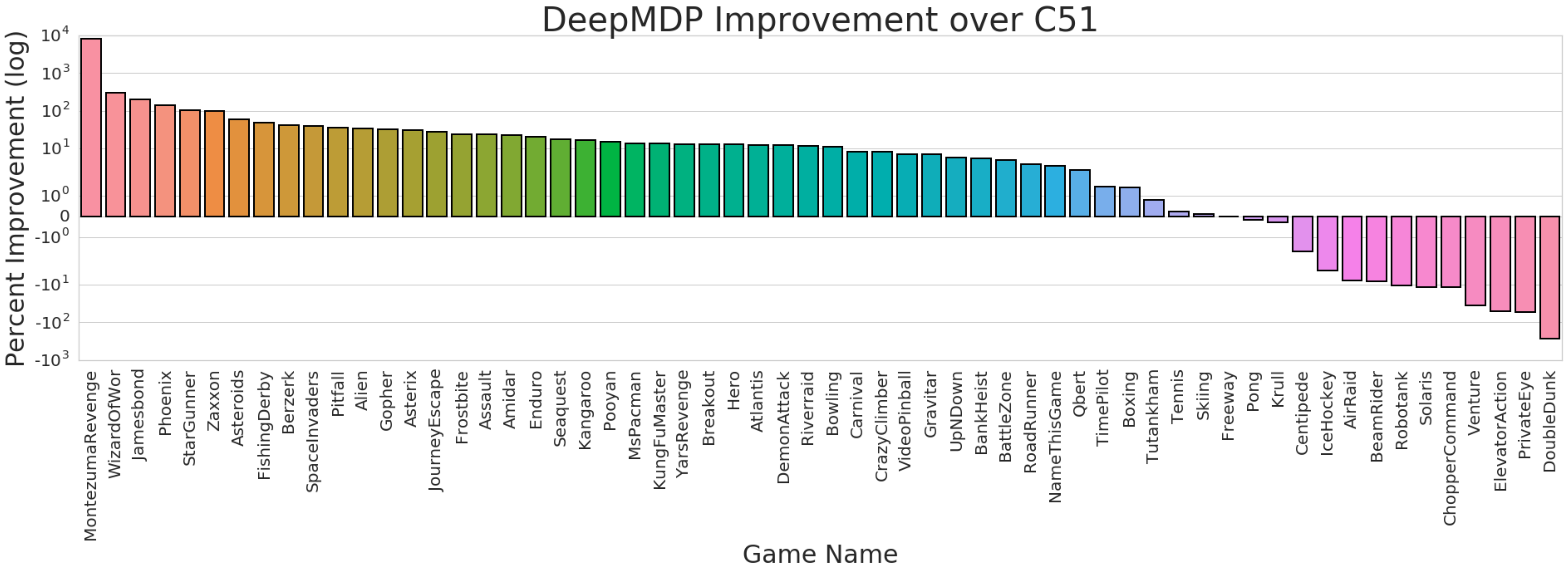}
    \caption{We compare the DeepMDP agent versus the C51 agent on the 60 games from the ALE (3 seeds each). For each game, the percentage performance improvement of DeepMDP over C51 is recorded.}
    \label{fig:performance_improvement}
\end{figure*}

\subsection{DonutWorld Experiments}

In order to evaluate whether we can learn effective representations, we study the representations learned by DeepMDPs in a simple synthetic environment we call \emph{DonutWorld}. DonutWorld consists of an agent rewarded for running clockwise around a fixed track. Staying in the center of the track results in faster movement. Observations are given in terms of 32x32 greyscale pixel arrays, but there is a simple 2D latent state space (the x-y coordinates of the agent). We investigate whether the x-y coordinates are correctly recovered when learning a two-dimensional representation.

This task epitomizes the low-dimensional dynamics, high-dimensional observations structure typical of Atari 2600 games, while being sufficiently simple to experiment with. We implement the DeepMDP training procedure using Tensorflow and compare it to a simple autoencoder baseline. See Appendix \ref{appendix:donut} for a full environment specification, experimental setup, and additional experiments. Code for replicating all experiments is included in the supplementary material.

In order to investigate whether the learned representations learned correspond well to reality, we plot a heatmap of closeness of representation for various states. Figure \ref{fig:toyheatmaps} shows that the DeepMDP representations effectively recover the underlying state of the agent, i.e. its 2D position, from the high-dimensional pixel observations. In contrast, the autoencoder representations are less meaningful, even when the autoencoder solves the task near-perfectly.

In Figure \ref{fig:multitoyheatmaps}, we modify the environment: rather than a single track, the environment now has four identical tracks. The agent starts in one uniformly at random and cannot move between tracks. The DeepMDP hidden state correctly merges all states with indistinguishable value functions, learning a deep state representation which is almost completely invariant to which track the agent is in.

The DeepMDP training loss can be difficult to optimize, as illustrated in Figure \ref{fig:competeloss}. This is due to the tendency of the transition and reward losses to compete with one another. If the deep state representation is uniformly zero, the transition loss will be zero as well; this is an easily-discovered local optimum, and gradient descent tends to arrive at this point early on in training. Of course, an informationless representation results in a large reward loss. As training progresses, the algorithm incurs a small amount of transition loss in return for a large decrease in reward loss, resulting in a net decrease in loss.

In DonutWorld, which has very simple dynamics, gradient descent is able to discover a good representation after only a few thousand iterations. However, in complex environments such as Atari, it is often much more difficult to discover representations that allow us to escape the low-information local minima. Using architectures with good inductive biases can help to combat this, as shown in Section \ref{sec:architectures}. This issue also motivates the use of auxiliary losses (such as value approximation losses or reconstruction losses), which may help guide the optimizer towards good solutions; see Appendix \ref{sec:appendix-atari-rep}.


\subsection{Atari 2600 Experiments}

\begin{figure*}[t]
    \centering
    \includegraphics[keepaspectratio, width=1.\textwidth]{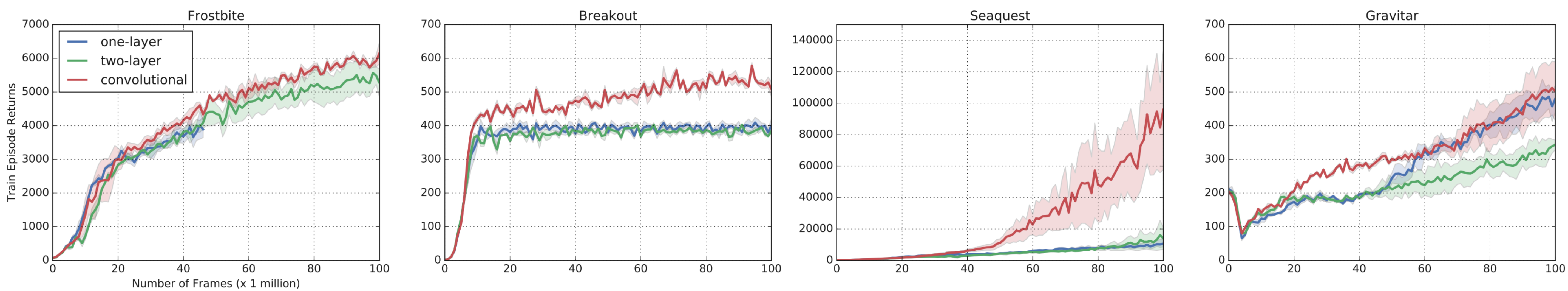}
    \caption{Performance of C51 with model-based auxiliary objectives. Three types of transition models are used for predicting next latent states: a single convolutional layer (convolutional), a single fully-connected layer (one-layer), and a two-layer fully-connected network (two-layer).}
    \label{fig:architectures}
\end{figure*}

\begin{figure*}[t]
    \centering
    \includegraphics[keepaspectratio, width=1.\textwidth]{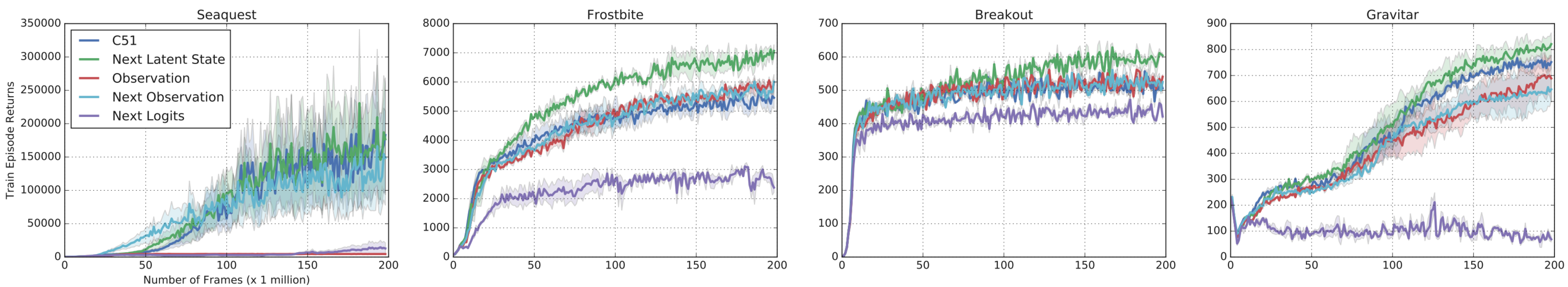}
    \caption{Using various auxiliary tasks in the Arcade Learning Environment. We compare predicting the next state's representation (Next Latent State, recommended by theoretical bounds on DeepMDPs) with reconstructing the current observation (Observation), predicting the next observation (Next Observation), and predicting the next C51 logits (Next Logits). Training curves for a baseline C51 agent are also shown.}
    \label{fig:related_work}
\end{figure*}


In this section, we demonstrate practical benefits of approximately learning a DeepMDP in the Arcade Learning Environment \cite{bellemare13arcade}. Our results on \textit{representation-similarity} indicate that learning a DeepMDP is a principled method for learning a high-quality representation. Therefore, we minimize DeepMDP losses as an auxiliary task alongside model-free reinforcement learning, learning a single representation which is shared between both tasks. Our implementations of the proposed algorithms are based on Dopamine \cite{bellemaredopamine}.


We adopt the Distributional Q-learning approach to model-free RL; specifically, we use as a baseline the C51 agent \cite{bellemare17distributional}, which estimates probability masses on a discrete support and minimizes the KL divergence between the estimated distribution and a target distribution.
C51 encodes the input frames using a convolutional neural network $\phi : \sspace \rightarrow \dsspace$, outputting a dense vector representation $\widebar{s} = \phi(s)$. The C51 Q-function is a feed-forward neural network which maps $\widebar{s}$ to an estimate of the reward distribution's logits.

To incorporate learning a DeepMDP as an auxiliary learning objective, we define a deep reward function and deep transition function. These are each implemented as a feed-forward neural network, which uses $\widebar{s}$ to estimate the immediate reward and the next-state representation, respectively. The overall objective function is a simple linear combination of the standard C51 loss and the Wasserstein distance-based approximations to the local DeepMDP loss given by Equations \ref{eqn:LocalRLoss} and \ref{eqn:LocalPLoss}. For experimental details, see Appendix \ref{appendix:atari}.

By optimizing $\phi$ to jointly minimize both C51 and DeepMDP losses, we hope to learn meaningful $\widebar{s}$ that form the basis for learning good value functions. In the following subsections, we aim to answer the following questions: (1) What deep transition model architecture is conducive to learning a DeepMDP on Atari? (2) How does the learning of a DeepMDP affect the overall performance of C51 on Atari 2600 games? (2) How do the DeepMDP objectives compare with similar representation-learning approaches?

\subsection{Transition Model Architecture}
\label{sec:architectures}
We compare the performance achieved by using different architectures for the DeepMDP transition model (see Figure \ref{fig:architectures}). We experiment with a single fully-connected layer, two fully-connected layers, and a single convolutional layer (see Appendix \ref{appendix:atari} for more details). We find that using a convolutional transition model leads to the best DeepMDP performance, and we use this transition model architecture for the rest of the experiments in this paper. 
Note how the performance of the agent is highly dependent on the architecture. We hypothesize that the inductive bias provided via the model has a large effect on the learned DeepMDPs.
Further exploring model architectures which provide inductive biases is a promising avenue to develop better auxiliary tasks. Particularly, we believe that exploring attention \citep{Vaswani2017AttentionIA, bahdanau2014neural} and relational inductive biases \citep{Watters2017VisualIN, Battaglia2016InteractionNF} could be useful in visual domains like Atari2600.

\subsection{DeepMDPs as an Auxiliary Task}\label{ssec:c51vsdmdp}
We show that when using the best performing DeepMDP architecture described in Appendix $\ref{ssec:architectures}$, we obtain nearly consistent performance improvements over C51 on the suite of 60 Atari 2600 games  (see Figure \ref{fig:performance_improvement}). 
 

\subsection{Comparison to Alternative Objectives}\label{ssec:mbrlcompare}
We empirically compare the effect of the DeepMDP auxilliary objectives on the performance of a C51 agent to a variety of alternatives. 
In the experiments in this section, we replace the deep transition loss suggested by the DeepMDP bounds with each of the following: 

(1) \textit{Observation Reconstruction}: We train a state decoder to reconstruct observations $s \in \sspace$ from $\widebar{s}$. This framework is similar to \cite{ha2018recurrent}, who learn a latent space representation of the environment with an auto-encoder, and use it to train an RL agent.

(2) \textit{Next Observation Prediction}: We train a transition model to predict next observations $s' \sim \prob(\cdot |s,a)$ from the current state representation $\widebar{s}$. This framework is similar to model-based RL algorithms which predict future observations \cite{tengyuma2018mbrl}.

(3) \textit{Next Logits Prediction}: We train a transition model to predict next-state representations such that the Q-function correctly predicts the logits of $(s', a')$, where $a'$ is the action associated with the max Q-value of $s'$. This can be understood as a distributional analogue of the Value Prediction Network, VPN, \cite{oh2017value}. Note that this auxiliary loss is used to update only the parameters of the representation encoder and the transition model, not the Q-function.

Our experiments demonstrate that the deep transition loss suggested by the DeepMDP bounds (i.e. predicting the next state's representation) outperforms all three ablations (see Figure \ref{fig:related_work}). Accurately modeling Atari 2600 frames, whether through observation reconstruction or next observation prediction, forces the representation to encode irrelevant information with respect to the underlying task. VPN-style losses have been shown to be helpful when using the learned predictive model for planning \citep{oh2017value}; however, we find that with a distributional RL agent, using this as an auxiliary task tends to hurt performance.

\section{Discussion on Model-Based RL}
We have focused on the implications of DeepMDPs for representation learning, but our results also provide a principled basis for model-based RL -- in latent space or otherwise.
Although DeepMDPs are latent space models, by letting $\phi$ be the identity function, all the provided results immediately apply to the standard model-based RL setting, where the model predicts states instead of latent states.
In fact, our results serve as a theoretical justification for common practices already found in the model-based deep RL literature. For example, \citet{chua2018deep,doerr2018probabilistic,hafner2018learning,buesing2018learning,feinberg2018model,Buckman2018SampleEfficientRL} train models to predict a reward and a distribution over next states, minimizing the negative log-probability of the true next state. The negative log-probability of the next state can be viewed as a one-sample estimate of the KL between the model's state distribution and the next state distribution. 
Due to Pinsker's inequality (which bounds the TV with the KL), and the suitability of TV as a metric (Section~\ref{sec:other_metrics}), this procedure can be interpreted as training a DeepMDP. Thus, the learned model will obey our local value difference bounds (Lemma~\ref{lem:localValueDiffBoundGeneral}) and suboptimality bounds (Theorem~\ref{thm:suboptGeneral}), which provide theoretical guarantees for the model.


Further, the suitability of Norm-MMD metrics for learning models presents a promising new research avenue for model-based RL: to break away from the KL and explore the vast family of Norm Maximum Mean Discrepancy metrics.

\section{Conclusions}
We introduce the concept of a DeepMDP: a parameterized latent space model trained via the minimization of tractable losses.
Theoretical analysis provides guarantees on the quality of the value functions of the learned model when the latent transition loss is any member of the large family of Norm Maximum Mean Discrepancy metrics. When the Wasserstein metric is used, a novel connection to bisimulation metrics guarantees the set of parametrizable policies is highly expressive. Further, it's guaranteed that two states with different values for any of those policies will never be collapsed under the representation.
Together, these findings suggest that learning a DeepMDP with the Wasserstein metric is a theoretically sound approach to representation learning.
Our results are corroborated by strong performance on large-scale Atari 2600 experiments, demonstrating that minimizing the DeepMDP losses can be a beneficial auxiliary task in model-free RL.

Using the transition and reward models of the DeepMDP for model-based RL (e.g. planning, exploration) is a promising future research direction. Additionally, extending DeepMDPs to accommodate different action spaces or time scales from the original MDPs could be a promising path towards learning hierarchical models of the environment.

\section*{Acknowledgements}
The authors would like to thank Philip Amortila and Robert Dadashi for invaluable feedback on the theoretical results; Pablo Samuel Castro, Doina Precup, Nicolas Le Roux, Sasha Vezhnevets, Simon Osindero, Arthur Gretton, Adrien Ali Taiga, Fabian Pedregosa and Shane Gu for useful discussions and feedback.

\section*{Changes From ICML 2019 Proceedings}
This document represents an updated version of our work relative to the version published in ICML 2019. The major addition was the inclusion of the generalization to Norm-MMD metrics and associated math in Section~\ref{sec:other_metrics}. Lemma~\ref{lem:LipschitzValue} also underwent minor changes to its statements and proofs.
Additionally, some sections were partially rewritten, especially the discussion on bisimulation (Section~\ref{sec:bisimulation}), which was significantly expanded.

\bibliography{references}
\bibliographystyle{icml2019}

\clearpage
{
\appendix
\onecolumn
{\Large \bf Appendix}
\pdfoutput=1
\section{Proofs}\label{sec:Proofs}
\subsection{Lipschitz MDP}
\LipschitzValue*
\begin{proof}
Start by proving \textbf{1}. By induction we will show that a sequence of Q values $\dQ_n$ converging to $\dQstar$ are all Lipschitz, and that as $n\to\infty$, their Lipschitz norm goes to $\frac{\Kdrew}{1-\gamma{\Kdprob}}$. Let ${\dQ}_0({\ds},a) = 0, \forall{\ds}\in{\dsspace}, a\in{\aspace}$ be the base case. Define $\dQ_{n+1}({\ds},a) = {\drew}({\ds},a) + \gamma \expect_{ \ds' \sim {\dprob}(\cdot | {\ds}, a)}\left[\ \max_{a'} {\dQ}_n({\ds}', a') \right]$. It is a well known result that the sequence ${\dQ}_n$ converges to $\dQstar$. Now let ${\KdQ}_{,n} = \sup_{a\in{\aspace}, {\ds}_1 \neq {\ds}_2\in{\dsspace}} \frac{|{\dQ}_n({\ds}_1,a) -  {\dQ}_{n}({\ds}_2, a)|}{{\distds}({\ds}_1,{\ds}_2)}$ be the Lipschitz norm of ${\dQ}_n$. Clearly ${\KdQ}_{,0}=0$.  Then,
\begin{align*}
{\KdQ}_{,n+1} &= \sup_{a\in{\aspace}, {\ds}_1 \neq {\ds}_2\in{\dsspace}} \frac{|{\dQ}_{n+1}({\ds}_1,a) -  {\dQ}_{n+1}({\ds}_2, a)|}{{\distds}({\ds}_1,{\ds}_2)} \\
&\le \sup_{a\in{\aspace}, {\ds}_1 \neq {\ds}_2\in{\dsspace}} \frac{|{\drew}({\ds}_1,a) -  \drew({\ds}_2, a)|}{{\distds}({\ds}_1,{\ds}_2)}
+ \gamma \sup_{a\in{\aspace}, {\ds}_1 \neq {\ds}_2\in{\dsspace}}  \frac{|  \expect_{ {\ds}_1' \sim {\dprob}({\cdot} | {\ds}_1, a) }  {\dQ}_{n}({\ds}_1',a) - \expect_{ {\ds}_2' \sim {\dprob}({\cdot} | {\ds}_2, a) }  {\dQ}_{n}({\ds}_2', a)|}{{\distds}({\ds}_1,{\ds}_2)} \\
&= {\Kdrew} + \gamma \sup_{a\in{\aspace}, {\ds}_1 \neq {\ds}_2\in{\dsspace}} \frac{|   \expect_{ {\ds}_1' \sim {\dprob}({\cdot} | {\ds}_1, a) } {\dQ}_{n}({\ds}_1',a) -  \expect_{ {\ds}_2' \sim {\dprob}({\cdot} | {\ds}_2, a) }  {\dQ}_{n}({\ds}_2', a)|}{{\distds}({\ds}_1,{\ds}_2)} \\
&\le {\Kdrew} + \gamma {\KdQ}_{,n} \sup_{a\in{\aspace}, {\ds}_1 \neq {\ds}_2\in{\dsspace}}   \frac{\wass{{\dprob}(\cdot | \ds_2, a)}{{\dprob}(\cdot | {{\ds}}_2, a)}}{{{\distds}({\ds}_1,{\ds}_2)}}, \text{  (Using the fact that ${\dQ}_n$ is ${\KdQ}_{,n}$-Lipschitz by induction) } \\
&\le {\Kdrew} + \gamma {\KdQ}_{,n} {\Kdprob} \\
&\le \sum_{i=0}^{n-1}  (\gamma {\Kdprob})^i {\Kdrew} + (\gamma {\Kdprob})^n {\KdQ}_{,0} = \sum_{i=0}^{n-1}  (\gamma {\Kdprob})^i {\Kdrew},  \text{ (by expanding the recursion)}
\end{align*}
Thus, as $n\to\infty$, $K_{\dQstar}\le\frac{\Kdrew}{1-\gamma{\Kdprob}}$.

To prove \textbf{2} a similar argument can be used. The sequence ${\dV}_{n+1}(\ds,a) = {\drew}_{\dpi}(\ds) + \gamma {\expect}_{ \ds' \sim \dprob(\cdot | \ds, a)}\left[ {\dV}_n(\ds') \right]$ converges to $\dQdpi$ and the sequence of Lipschitz norms converge to $\frac{{\Kdrew}_{\dpi}}{1-\gamma \Kdprob}$. From there it's trivial to show that $\dQdpi$ is also Lipschitz.

Finally, we prove \textbf{3}. Note that the transition function of a constant policy $\pi(a)$ has the following property:
\begin{align*}
\wass{\dprob_{\dpi}(\cdot | \ds_2)}{\dprob_{\dpi}(\cdot | \ds_2)} &= 
\sup_{f \in  \F} \left| \int (\dprob_{\dpi}(\ds' | \ds_2) - \dprob_{\dpi}(\ds' | \ds_2)) f(\ds') \deriv \ds' \right| \\
&\le \sup_{f \in  \F} \left| \int \sum_a \pi(a)( \dprob(\ds' | \ds_2,a) - \dprob(\ds' | \ds_2,a)) f(\ds') \deriv \ds' \right|\\
&\le  \sum_a \pi(a)\sup_{f \in  \F} \left| \int ( \dprob(\ds' | \ds_2,a) - \dprob(\ds' | \ds_2,a)) f(\ds') \deriv \ds' \right| \\
&\le  \sum_a \pi(a) \Kdprob \\
&\le \Kdprob
\end{align*}
Similarly, $|\drew_{\dpi}({\ds}_1) - {\drew}_{\dpi}({\ds}_2)| \le \Kdrew$. Thus, for a constant policy $\pi$, the Lipschitz norms ${\Kdprob}_{\dpi}\le \Kdprob$ and ${\Kdrew}_{\dpi}\le \Kdrew$. To complete the proof we can apply result \textbf{2}.
\end{proof}

\subsection{Global DeepMDP}

\globalValueDifferenceBound*
\begin{proof}
This is a specific case to the general Lemma~\ref{lem:globalValueDifferenceBoundGeneral}
\end{proof}

\globalLipschitzRepresentationQuality*
\begin{proof}
\begin{align*}
\left | \Qdpi (s_1, a) - \Qdpi (s_2 ,a) \right | &\le
 \left | \dQdpi (s_1, a) - \dQdpi (s_2,a) \right | + \left | \Qdpi (s_1, a) - \dQdpi (s_1,a) \right | + \left | \Qdpi (s_2, a) - \dQdpi (s_2,a) \right |  \\
&\le  \left | \dQdpi (s_1, a) - \dQdpi (s_2,a) \right | +  2\frac{ \left( \gLdrew  + \gamma \KV \gLdprob \right ) }{1-\gamma}  \text{  Applying Lemma \ref{lem:globalValueDifferenceBound}} \\
&\le \KV \norm{\phi(s_1) - \phi(s_2)} + 2\frac{ \left( \gLdrew  + \gamma \KV \gLdprob \right ) }{1-\gamma} \text{  Using the Lipschitz property of $\dQdpi$}
\end{align*}
\end{proof}

\begin{restatable}{theorem}{globalSubOpt}\label{thm:globalSubOpt}
Let $\mdp$ and $\dmdp$ be an MDP and a $(\Kdrew,\Kdprob)$-Lipschitz DeepMDP respectively, with an embedding function $\phi$ and global loss functions $\gLdrew$ and $\gLdprob$. For all $s\in\sspace$, the suboptimality of the optimal policy $\dpistar$ of $\dmdp$ evaluated on $\mdp$ can be bounded by:
\begin{equation*}
    \Vstar(s) - V^{\dpistar}(s)
    \le 2\frac{\gLdrew}{1-\gamma} 
    + 2\gamma\frac{ \Kdrew \gLdprob}{(1-\gamma)(1 - \gamma\Kdprob)}
\end{equation*}
\end{restatable}
\begin{proof}
This is just a case of the general Theorem~\ref{thm:suboptGeneral} combined with the result that the optimal policy of a $(\Kdrew,\Kdprob)$-Lipschitz DeepMDP is $\frac{\Kdrew}{1-\gamma\Kdprob}$-Lipschitz-valued.
\end{proof}

\subsection{Local DeepMDP}
\localValueDiffBoundW*
\begin{proof}
This Lemma is just an example of the general Lemma~\ref{lem:localValueDiffBoundGeneral}
\end{proof}

\localLipschitzRepresentationQuality*
\begin{proof}
$\expddpi$

Using the fact that $\left| { {\Vdpi}(s) - {\dVdpi}(s) } \right| \le d_{\dpi}^{-1}(s) {\expddpi} \left| {{\Vdpi}(s) - {\dVdpi}(s)}\right|$,
\begin{align*}
\left| \Vdpi(s_1) - \Vdpi(s_2) \right| &\le \left|\dVdpi(s_1) - \dVdpi(s_2)\right|
+ d_{\dpi}^{-1}(s_1) \expddpi \left| \Vdpi(s_1) - \dVdpi(s_1) \right|
+ d_{\dpi}^{-1}(s_2) \expddpi \left| \Vdpi(s_2) - \dVdpi(s_2) \right|
 \\
&\le \left|\dVdpi(s_1) - \dVdpi(s_2)\right|
+ \frac{
 \lLdrew  + \gamma \KV \lLdprob
}{1-\gamma}
\left ( d_{\dpi}^{-1}(s_1) + d_{\dpi}^{-1}(s_2)  \right ), \text{  Applying Lemma~\ref{lem:localValueDiffBoundW}} \\
&\le {\KdV} \distds(\phi(s_1), \phi(s_2))
+ \frac{
 {\lLdrew}  + \gamma {\KV} {\lLdprob}
}{1-\gamma}
\left ( d_{\dpi}^{-1}(s_1) + d_{\dpi}^{-1}(s_2)  \right )
\end{align*}
\end{proof}

\subsection{Connection to Bisimulation}
\begin{restatable}{lemma}{BisimulationMetricUpperBound}\label{lem:BisimulationMetricUpperBound}
Let $\dmdp$ be a $\Krew$-$\Kprob$-Lipschitz DeepMDP with a metric in the state space $\distds$. Then the bisimulation metric $\td$ is Lipschitz s.t. $ \forall \ds_1,\ds_2\in\dsspace$,
\begin{equation}
\td(\ds_1, \ds_2) \le \frac{(1-\gamma)\Kdrew}{1 - \gamma \Kdprob} \distds(\ds_1, \ds_2).
\end{equation}
\end{restatable}
\begin{proof}
We first derive a property of the Wasserstein. Let $d$ and $p$ be pseudometrics in $\chi$ s.t. $d(x,y) \le K p(x,y)$ for all $x,y\in\chi$, and let $P$ and $Q$ be two distributions in $\chi$. Then $\W_d(P,Q) \le K \W_p(P,Q)$. To prove it, first note that define the sets of $C$-Lipschitz functions for both metrics:
\begin{align*}
&\F_{d,C} = \left\{f: \forall x\neq y \in\chi,  | f(x) - f(y) | \le C d(x,y) \right \}, \\
&\F_{p,C} = \left\{f: \forall x\neq y \in\chi,  | f(x) - f(y) | \le C p(x,y) \right \}.
\end{align*}
Then it becomes clear that $\F_{d,1} \underline{\subset} \F_{p,K}$. We can now prove the property:
\begin{align*}
\W_{d}( P, Q ) &= \sup_{f \in \F_{d,1}} \left | \expect_{x \sim P} f(x) - \expect_{y \sim Q} f(y) \right | \\
&\le \sup_{f \in \F_{p,K}} \left | \expect_{x \sim P} f(x) - \expect_{y \sim Q} f(y) \right | \\
&= \sup_{f \in \F_{p,1}} \left | \expect_{x \sim P} Kf(x) - \expect_{y \sim Q} Kf(y) \right | \\
&= K \sup_{f \in \F_{p,1}} \left | \expect_{x \sim P} f(x) - \expect_{y \sim Q} f(y) \right | \\
&= K \W_{p}( P, Q )
\end{align*}
We prove the Lemma by induction. We show that a sequence of pseudometrics values $d_n$ converging to $\td$ are all Lipschitz, and that as $n\to\infty$, their Lipschitz norm goes to $\frac{(1-\gamma) \Kdrew}{1-\gamma\Kdprob}$. Let $d_0(s_1, s_2) = 0, \forall s_1,s_2\in\sspace$ be the base case. Define $d_{n+1}(s_1,s_2) = F_d(s_1,s_2)$ as defined in Definition~\ref{def:bisimulation_metric}. \citet{ferns2011bisimulation} shows that $F$ is a contraction, and that $d_n$ converges to $\td$ as $n\to\infty$. Now let $K_{d,n} = \sup\limits_{s_1 \neq s_2\in\sspace} \frac{ d_n(s_1,s_2) }{\dists(s_1,s_2)}$ be the Lipschitz norm of $d_n$. Also see that $K_{d,0}=0$.  Then,
\begin{align*}
K_{d,n+1} &= \sup_{s_1 \neq s_2\in\sspace} \frac{ d_{n+1}(s_1,s_2) }{\dists(s_1,s_2)} \\
&= \sup_{s_1 \neq s_2\in\sspace} \frac{ F_{d_{n}}(s_1,s_2) }{\dists(s_1,s_2)} \\
&\le   (1-\gamma) \sup_{a\in\aspace, s_1 \neq s_2\in\sspace} \frac{ |\rew(s_1,a) - \rew(s_2, a)| }{\dists(s_1,s_2)} + \gamma  \sup_{a\in\aspace, s_1 \neq s_2\in\sspace} \frac{  \W_{d_n} \left( \prob(\cdot | s_2, a), \prob(\cdot | s_2, a) \right ) }{\dists(s_1,s_2)} \\
&=   (1-\gamma) \Krew + \gamma  \sup_{a\in\aspace, s_1 \neq s_2\in\sspace} \frac{  \W_{d_n} \left( \prob(\cdot | s_2, a), \prob(\cdot | s_2, a) \right ) }{\dists(s_1,s_2)} \\
&\le   (1-\gamma) \Krew + \gamma K_{d,n} \sup_{a\in\aspace, s_1 \neq s_2\in\sspace} \frac{  \W_{\dists} \left( \prob(\cdot | s_2, a), \prob(\cdot | s_2, a) \right ) }{\dists(s_1,s_2)}, \text{  (Using the property derived above.) } \\
&\le   (1-\gamma) \Krew + \gamma K_{d,n} \Kprob \\ 
&\le  (1-\gamma) \sum_{i=0}^{n-1}  (\gamma \Kprob)^i \Krew,  \text{ (by expanding the recursion)}
\end{align*}
Thus, even as $n\to\infty$, $K_{d,n}\le\frac{(1-\gamma) \Kdrew}{1-\gamma\Kdprob}$.
\end{proof}

\begin{restatable}{lemma}{PhiBisimulationMetricUpperBound}\label{lem:PhiBisimulationMetricUpperBound}
Let $\mdp$ be an MDP and $\dmdp$ be a $\Kdrew$-$\Kdprob$-Lipschitz MDP with an embedding function $\phi: \sspace \to \dsspace$ and global DeepMDP losses $\gLdprob$ and $\gLdrew$.. We can extend the bisimulation metric to also measure a distance between $s \in \sspace$ and $\ds \in \dsspace$ by considering an joined MDP constructed by joining $\mdp$ and $\dmdp$. When an action is taken, each state will transition according to the transition function of its corresponding MDP. Then the bisimulation metric between a state $s\in\sspace$ and it's embedded counterpart $\phi(s)$ is bounded by:
\begin{equation*}
\td(s, \phi(s)) \le \gLdrew + \gamma \gLdprob \frac{\Kdrew}{1 - \gamma \Kdprob}
\end{equation*}
\end{restatable}
\begin{proof}
First, note that
\begin{align*}
\W_{\td} (\phi \prob(\cdot | s, a), \dprob(\cdot | \phi(s), a)) &= \sup_{f \in \mathbb{F}_{\td}} \expect_{\ds_1' \sim \phi \prob(\cdot | s, a)} [f(\ds_1')] - \expect_{\ds_2' \sim \dprob(\cdot | \phi(s), a)} [f(\ds_2')] \\
&\le  \frac{(1-\gamma)K_{\rew}}{1 - \gamma K_{\prob}} \sup_{f \in \mathbb{F}_1} \expect_{\ds_1' \sim \phi \prob(\cdot | s, a)} [f(\ds_1')] - \expect_{\ds_2' \sim \dprob(\cdot | \phi(s), a)} [f(\ds_2')]  &  \text{(Using Theorem \ref{lem:BisimulationMetricUpperBound})} \\
&= \frac{(1-\gamma)K_{\rew}}{1 - \gamma K_{\prob}} \W_{\ell_2}(\prob(\cdot | s, a), \dprob(\cdot | \phi(s), a)) \\
&\le \frac{(1-\gamma)K_{\rew}}{1 - \gamma K_{\prob}} \gLdprob
\end{align*}
Using the triangle inequality of pseudometrics and the previous derivation:
\begin{align*}
\sup_{s} \td(s, \phi(s)) &= \max_{a \in \aspace} \left ( (1 - \gamma) \left| \rew(s, a) - \drew(\phi(s),a) \right | + \gamma \W_{\td} (\prob(\cdot | s, a), \dprob(\cdot | \phi(s), a)) \right ) \\ 
&\le (1 - \gamma) \gLdrew + \gamma  \max_{a \in \aspace} \left (  \W_{\td} (\prob(\cdot | s, a), \phi \prob(\cdot | s, a)) + \W_{\td} (\phi \prob(\cdot | s, a), \dprob(\cdot | \phi(s), a))  \right ) \\ 
&\le (1 - \gamma) \gLdrew + \gamma \frac{(1-\gamma)K_{\rew}}{1 - \gamma K_{\prob}} \gLdprob + \gamma \max_{a \in \aspace}   \W_{\td} (\prob(\cdot | s, a), \phi \prob(\cdot | s, a))  \\ 
&\le (1 - \gamma) \gLdrew + \gamma \frac{(1-\gamma)K_{\rew}}{1 - \gamma K_{\prob}} \gLdprob + \gamma \sup_{s} \td(s' , \phi(s))
\end{align*}
Solving for the recurrence leads to the desired result.
\end{proof}

\BisimulationMetricUpperBoundPhi*
\begin{proof}
\begin{align*}
\td(s_1, s_2)  &\le \td(s_1, \phi(s_1)) + \td(s_2, \phi(s_2)) + \td(\phi(s_1), \phi(s_2)) \\
&\le  2 \left ( \gLdrew + \gamma \gLdprob \frac{\Kdrew}{1 - \gamma \Kdprob} \right ) + \td(\phi(s_1), \phi(s_2)) & \text{(Using Theorem \ref{lem:PhiBisimulationMetricUpperBound})} \\
&\le  2 \left (  \gLdrew + \gamma \gLdprob \frac{\Kdrew}{1 - \gamma \Kdprob} \right ) + \frac{(1-\gamma)\gLdrew}{1 - \gamma K_{\dprob}} \norm{\phi(s_1) - \phi(s_2)} & \text{(Applying Theorem \ref{lem:BisimulationMetricUpperBound})}
\end{align*}
Completing the proof.
\end{proof}

\subsection{Quality of $\dPi$}

\begin{lemma}\label{lem:LipschitzFunctionEpsilon}
Let $d_f$ and $d_g$ be the metrics on the space $\chi$, with the property that for some $\epsilon \ge 0$ it holds that $\forall x,y \in \chi,   d_f(x,y) \le \epsilon + d_g(x,y)$. Define the sets of $1$-Lipschitz functions $\mathbb{F} = \left\{ f : |f(x) - f(y) | \le d_f(x,y), \forall x,y\in\chi \right\}$ and $\mathbb{G} = \left\{ g : |g(x) - g(y) | \le d_g(x,y), \forall x,y\in\chi \right\}$. Then for any $f \in \mathbb{F}$, there exists one $g \in \mathbb{G}$ such that for all $x\in\chi$,
\begin{equation*}
|f(x) - g(x) | \le \frac{\epsilon}{2}
\end{equation*}
\end{lemma}
\begin{proof}
Define the set  $\mathbb{Z} = \left\{ z : |z(x) - z(y) | \le \epsilon + d_g(x,y), \forall x,y\in\chi \right\}$. Then trivially, any function $f \in \mathbb{F}$ is also a member of $\mathbb{Z}$. We now show that the set $\mathbb{Z}$ can equivalently be expressed as $z(x) = g(x) + u(x)$, where $g \in \mathbb{G}$ and $u(x) \in (\frac{-\epsilon}{2}, \frac{\epsilon}{2})$, is  (non Lipschitz) bounded function.
\begin{align*}
|z(x) - z(y)| &= | g(x) + u(x) - g(y) - u(y) | \\
&\le | g(x) - g(y) | + | u(x) - u(y) | \\
&\le d_g(x,y) + \epsilon
\end{align*}
Note how both inequalities are tight (there is a $g$ and $u$ for which the equality holds), together with the fact that the set $\mathbb{Z}$ is convex, it follows that any $z\in\mathbb{Z}$ must be expressible as $g(x) + u(x)$.

We now complete the proof. For any $z\in\mathbb{Z}$, there exist a $g\in\mathbb{G}$ s.t. $z(x) = g(x) + u(x)$. Then:
\begin{equation*}
|z(x) - g(x) | = |u(x)| \le \frac{\epsilon}{2}
\end{equation*}
\end{proof}

\dPiBisim*
\begin{proof}The proof is based on Lemma~\ref{lem:LipschitzFunctionEpsilon}. Let $\chi = \sspace$, $d_f(x, y) = K \td(x,y)$, $d_g(x,y)=KC\norm{\phi(x) - \phi(y)}$ and $\epsilon = 2 \left(  \gLdrew + \gamma \gLdprob \frac{\Kdrew}{1 - \gamma \Kdprob} \right)$. Theorem~\ref{thm:BisimulationMetricUpperBoundPhi} can be used to show that the condition $d_f(x,y) \le \epsilon + d_g(x,y)$ holds. Then the application of Lemma~\ref{lem:LipschitzFunctionEpsilon} provides the desired result.
\end{proof}

\subsection{Generalized Value Difference Bounds}\label{sec:generalizedProofs}

\begin{restatable}{lemma}{globalValueDifferenceBoundGeneral}\label{lem:globalValueDifferenceBoundGeneral}
Let $\mdp$ and $\dmdp$ be an MDP and DeepMDP respectively, with an embedding function $\phi$ and global loss functions $\gLdrew$ and $\gLdprob$, where $\gLdprob$ uses on a Norm-MMD $\D$. For any $\KdV$-smooth-valued policy $\dpi \in \dPi$ as in Definition~\ref{defn:SmoothValue}. The value difference can be bounded by
\begin{equation*}
\left | \Qdpi (s, a) - \dQdpi (\phi(s) ,a) \right |  \leq 
\frac{  \gLdrew  + \gamma \KdV \gLdprob  }{1-\gamma}.
\end{equation*}
\end{restatable}
\begin{proof}
The proof consists of showing that the supremum $\sup_{s,a}\left | \Qdpi (s, a) - \dQdpi (\phi(s) ,a) \right |$ is bounded by a recurrence relationship.
\begin{align*}
\sup_{s\in\sspace,a\in\aspace}\left | \Qdpi (s, a) - \dQdpi (\phi(s) ,a) \right | &\leq \sup_{s\in\sspace,a\in\aspace} \left | \rew (s, a) - \drew (\phi(s) ,a) \right | 
+\gamma \sup_{s\in\sspace,a\in\aspace}  \left | \expect_{s' \sim \prob(\cdot | s, a)}   \Vdpi (s') - \expect_{\ds' \sim \dprob(\cdot | \phi(s), a)}  \dVdpi (\ds') \right | \\
&= \gLdrew
+\gamma \sup_{s\in\sspace,a\in\aspace}  \left | \expect_{s' \sim \prob(\cdot | s, a)} 
[   \Vdpi (s') - \dVdpi (\phi(s')) ] +
\expect_{\substack{\ds' \sim \dprob(\cdot | \phi(s), a) \\ s' \sim \prob(\cdot | s, a)}} 
[ \dVdpi  (\phi(s')) - \dVpi(\ds') ] \right | \\
&\le \gLdrew
+\gamma \sup_{s\in\sspace,a\in\aspace}  \left | \expect_{s' \sim \prob(\cdot | s, a)} 
[   \Vdpi (s') - \dVdpi (\phi(s')) ] \right | +
\gamma \sup_{s\in\sspace,a\in\aspace}  \left | \expect_{\substack{\ds' \sim \dprob(\cdot | \phi(s), a) \\ s' \sim \prob(\cdot | s, a)}} 
[ \dVdpi  (\phi(s')) - \dVpi(\ds') ] \right | \\
&\le \gLdrew
+\gamma \sup_{s\in\sspace,a\in\aspace}  \left | \expect_{s' \sim \prob(\cdot | s, a)} 
[   \Vdpi (s') - \dVdpi (\phi(s')) ] \right | +
\gamma \KdV \sup_{s\in\sspace,a\in\aspace} \D \left( \phi\prob(\cdot | s, a), \dprob(\cdot | \phi(s), a) \right) \\
&= \gLdrew
+\gamma \sup_{s\in\sspace,a\in\aspace}  \left | \expect_{s' \sim \prob(\cdot | s, a)} 
[   \Vdpi (s') - \dVdpi (\phi(s')) ] \right | +
\gamma \KdV \gLdprob\\
&\le \gLdrew
+\gamma \sup_{s\in\sspace,a\in\aspace}  \expect_{s' \sim \prob(\cdot | s, a)}  \left |
[   \Vdpi (s') - \dVdpi (\phi(s')) ] \right | +
\gamma \KdV \gLdprob \text{  Using Jensen's inequality.}\\
&\le \gLdrew
+\gamma \sup_{s\in\sspace,a\in\aspace}  \left |
[   \Vdpi (s) - \dVdpi (\phi(s)) ] \right | +
\gamma \KdV \gLdprob \\
&\le \gLdrew
+\gamma \sup_{s\in\sspace,a\in\aspace}  \left |
[   \Qdpi (s,a) - \dQdpi (\phi(s),a) ] \right | + 
\gamma \KdV \gLdprob 
\end{align*}
Solving for the recurrence relation over $\sup_{s\in\sspace,a\in\aspace}  \left | \Qdpi (s, a) - \dQdpi (\phi(s) ,a) \right |$ results in the desired result.
\end{proof}

\begin{restatable}{lemma}{localValueDiffBoundGeneral}\label{lem:localValueDiffBoundGeneral}
Let $\mdp$ and $\dmdp$ be an MDP and DeepMDP respectively, with an embedding function $\phi$ and let $\D$ be a Norm-MMD metric. For any $\KdV$-smooth-valued policy $\dpi \in \dPi$ (as in Definition~\ref{defn:SmoothValue}), let $\lLdrew$ and $\lLdprob$ be the local loss functions measured under $\xidpi$, the stationary state action distribution of $\dpi$. Then the value difference can be bounded by:
\begin{multline*}
\expxidpi \left | \Qdpi (s, a) - \dQdpi (\phi(s) ,a) \right |   \leq 
\frac{ \lLdrew  + \gamma \KdV \lLdprob }{1-\gamma} ,
\end{multline*}
\end{restatable}
\begin{proof}
\begin{align*}
\expxidpi  \left | \Qdpi (s, a) - \dQdpi (\phi(s) ,a) \right | &\leq \expxidpi  \left | \rew (s, a) - \drew (\phi(s) ,a) \right | 
+\gamma \expxidpi  \left | \expect_{s' \sim \prob(\cdot | s, a)}   \Vdpi (s') - \expect_{\ds' \sim \dprob(\cdot | \phi(s), a)}  \dVdpi (\ds') \right | \\
&= \lLdrew
+\gamma \expxidpi  \left | \expect_{s' \sim \prob(\cdot | s, a)} 
[   \Vdpi (s') - \dVdpi (\phi(s')) ] +
\expect_{\substack{\ds' \sim \dprob(\cdot | \phi(s), a) \\ s' \sim \prob(\cdot | s, a)}} 
[ \dVdpi  (\phi(s')) - \dVpi(\ds') ] \right | \\
&\le \lLdrew
+\gamma \expxidpi  \left | \expect_{s' \sim \prob(\cdot | s, a)} 
[   \Vdpi (s') - \dVdpi (\phi(s')) ] \right | +
\gamma \expxidpi  \left | \expect_{\substack{\ds' \sim \dprob(\cdot | \phi(s), a) \\ s' \sim \prob(\cdot | s, a)}} 
[ \dVdpi  (\phi(s')) - \dVpi(\ds') ] \right | \\
&\le \lLdrew
+\gamma \expxidpi  \left | \expect_{s' \sim \prob(\cdot | s, a)} 
[   \Vdpi (s') - \dVdpi (\phi(s')) ] \right | +
\gamma \KdV \expxidpi \D \left( \phi\prob(\cdot | s, a), \dprob(\cdot | \phi(s), a) \right ) \\
&= \lLdrew
+\gamma \expxidpi  \left | \expect_{s' \sim \prob(\cdot | s, a)} 
[   \Vdpi (s') - \dVdpi (\phi(s')) ] \right | +
\gamma \KdV \lLdprob\\
&\le \lLdrew
+\gamma \expxidpi  \expect_{s' \sim \prob(\cdot | s, a)}  \left |
[   \Vdpi (s') - \dVdpi (\phi(s')) ] \right | +
\gamma \KdV \lLdprob \text{ Using Jensen's inequality.}\\
&\le \lLdrew
+\gamma \expxidpi  \left |
[   \Vdpi (s) - \dVdpi (\phi(s)) ] \right | +
\gamma \KdV \lLdprob \text{ Applying the stationarity property.}\\
&\le \lLdrew
+\gamma \expxidpi  \left |
[   \Qdpi (s,a) - \dQdpi (\phi(s),a) ] \right | +
\gamma \KdV \lLdprob 
\end{align*}
Solving for the recurrence relation over $\expxidpi  \left | \Qdpi (s, a) - \dQdpi (\phi(s) ,a) \right |$ results in the desired result.
\end{proof}

\begin{restatable}{theorem}{suboptGeneral}\label{thm:suboptGeneral}
Let $\mdp$ and $\dmdp$ be an MDP and a $(K_R,K_P)$-Lipschitz DeepMDP respectively, with an embedding function $\phi$ and global loss functions $\gLdrew$ and $\gLdprob$. For all $s\in\sspace$, the suboptimality of the optimal policy $\dpistar$ of $\dmdp$ evaluated on $\mdp$ can be bounded by,
\begin{equation*}
    \Vstar(s) - V^{\dpistar}(s)
    \le 2\frac{\gLdrew  + \gamma \normD{\dVstar} \gLdprob }{1-\gamma}
\end{equation*}
Where $\normD{\dVstar}$ is the smoothness of the optimal value function.
\end{restatable}
\begin{proof}
For any $s\in\sspace$ we have
\begin{equation}
    \label{eq:v-triangle}
    |\Vstar(s) - V^{\dpistar}(s)| \le |\dVstar(\phi(s)) - V^{\dpistar}(s)| + |\Vstar(s) - \dVstar(\phi(s))|.
\end{equation}
Using the result given by Lemma~\ref{lem:globalValueDifferenceBoundGeneral}, we may bound the first term of the RHS by $\frac{  \gLdrew  + \gamma \normD{\dVstar} \gLdprob }{1-\gamma}$.

To bound the second therm, we first show that for any $s\in\sspace,a\in\aspace$, we have,
\begin{align}
    \label{eq:qstar-diff}
    |\Qstar(s,a) - \dQstar(\phi(s),a)| ~\le~ & \frac{\gLdrew  + \gamma \normD{\dVstar} \gLdprob }{1-\gamma}.
\end{align}
We prove this similarly to Lemma~\ref{lem:globalValueDifferenceBound}:
\begin{align*}
\sup_{s\in\sspace,a\in\aspace}\left | \Qstar (s, a) - \dQstar (\phi(s) ,a) \right | &\leq \sup_{s\in\sspace,a\in\aspace} \left | \rew (s, a) - \drew (\phi(s) ,a) \right | 
+\gamma \sup_{s\in\sspace,a\in\aspace}  \left | \expect_{s' \sim \prob(\cdot | s, a)}   \Vstar (s') - \expect_{\ds' \sim \dprob(\cdot | \phi(s), a)}  \dVstar (\ds') \right | \\
&= \gLdrew
+\gamma \sup_{s\in\sspace,a\in\aspace}  \left | \expect_{s' \sim \prob(\cdot | s, a)} 
[   \Vstar (s') - \dVstar (\phi(s')) ] +
\expect_{\substack{\ds' \sim \dprob(\cdot | \phi(s), a) \\ s' \sim \prob(\cdot | s, a)}} 
[ \dVstar  (\phi(s')) - \dVstar(\ds') ] \right | \\
&\le \gLdrew
+\gamma \sup_{s\in\sspace,a\in\aspace}  \left | \expect_{s' \sim \prob(\cdot | s, a)} 
[   \Vstar (s') - \dVstar (\phi(s')) ] \right | +
\gamma \sup_{s\in\sspace,a\in\aspace}  \left | \expect_{\substack{\ds' \sim \dprob(\cdot | \phi(s), a) \\ s' \sim \prob(\cdot | s, a)}}
[ \dVstar  (\phi(s')) - \dVstar(\ds') ] \right | \\
&= \gLdrew
+\gamma \sup_{s\in\sspace,a\in\aspace}  \left | \expect_{s' \sim \prob(\cdot | s, a)} 
[   \Vstar (s') - \dVstar (\phi(s')) ] \right | +
\gamma \normD{\dVstar} \gLdprob\\
&\le \gLdrew
+\gamma \max_ {s}  \left |
  \Vstar (s) - \dVstar (\phi(s)) \right | +
\gamma \normD{\dVstar} \gLdprob \text{  Using Jensen's inequality.}\\
&= \gLdrew
+\gamma \max_{s}  \left | \max_a \Qstar (s,a) - \max_a\dQstar (\phi(s),a) \right | +
\gamma \normD{\dVstar} \gLdprob \\
&\le \gLdrew
+\gamma \sup_{s\in\sspace,a\in\aspace}  \left | \Qstar (s,a) - \dQstar (\phi(s),a) \right | +
\gamma \normD{\dVstar} \gLdprob 
\end{align*}
Solving for the recurrence gives the desired result. Then, the second therm can be easily bounded:
\begin{align}
    |\Vstar(s) - \dVstar(\phi(s))| = & |\max_a \Qstar(s, a) - \max_{a'} \dQstar(\phi(s), a')| \\
    \le & \max_a |\Qstar(s, a) - \dQstar(\phi(s), a)| \\ \le& \frac{ \gLdrew  + \gamma \normD{\dVstar} \gLdprob }{1-\gamma}.
\end{align}
as desired. Combining the bounds for the first and second terms completes the proof.
\end{proof}

\section{DonutWorld Experiments}\label{appendix:donut}

\subsection{Environment Specification}
Our synthetic environment, DonutWorld, consists of an agent moving around a circular track. The environment is centered at (0,0), and includes the set of points whose distance to the center is between 3 and 6 units away; all other points are out-of-bounds. The distance the agent can move on each timestep is equal to the distance to the nearest out-of-bounds point, capped at 1. We refer to the regions of space where the agent's movements are fastest (between 4 and 5 units away from the origin) as the ``track,'' and other in-bounds locations as ``grass''. Observations are given in the form of 32-by-32 black-and-white pixel arrays, where the agent is represented by a white pixel, the track by luminance 0.75, the grass by luminance 0.25, and out-of-bounds by black. The actions are given as pairs of numbers in the range (-1,1), representing an unnormalized directional vector. The reward for each transition is given by the number of radians moved clockwise around the center.

Another variant of this environment involves four copies of the track, all adjacent to one another. The agent is randomly placed onto one of the four tracks, and cannot move between them. Note that the value function for any policy is identical whether the agent is on the one-track DonutWorld or the four-track DonutWorld. Observations for the four-track DonutWorld are 64-by-64 pixel arrays.

\subsection{Architecture Details}
We learn a DeepMDP on states and actions from a uniform distribution over all possible state-action pairs. The environment can be fully represented by a latent space of size two, so that is the dimensionality used for latent states of the DeepMDP.

We use a convolutional neural net for our embedding function $\phi$, which contains three convolutional layers followed by a linear transformation. Each convolutional layer uses 4x4 convolutional filters with stride of 2, and depths are mapped to 2, then 4, then 8; the final linear transformation maps it to the size of the latent state, 2. ReLU nonlinearities are used between each layer, and a sigmoid is applied to the output to constrain the space of latent states to be bounded by (0, 1).

The transition function and reward function are each represented by feed-forward neural networks, using 2 hidden layers of size 32 with ReLU nonlinearities. A sigmoid is applied output of the transition function.

For the autoencoder baseline, we use the same architecture for the encoder as was used for the embedding function. Our decoder is a three-layer feedforward ReLU network with 32 hidden units per layer. The reconstruction loss is a softmax cross-entropy over possible agent locations.

\subsection{Hyperparameters}
All models were implemented in Tensorflow. We use an Adam Optimizer with a learning rate of 3e-4, and default settings. We train for 30,000 steps. The batch size is 256 for DMDPs and 1024 for autoencoders. The discount factor, $\gamma$, is set to 0.9, and the coefficient for the gradient penalty, $\lambda$, is set to 0.01. In contrast to the gradient penalty described in \citet{gulrajani2017improved}, which uses its gradient penalty to encourage all gradient norms to be close to 1, we encourage all gradient norms to be close to 0. Our sampling distribution is the same as our training distribution, simply the distribution of states sampled from the environment.

\subsection{Empirical Value Difference}

\begin{figure*}[t]
\centering
\includegraphics[keepaspectratio, width=\textwidth]{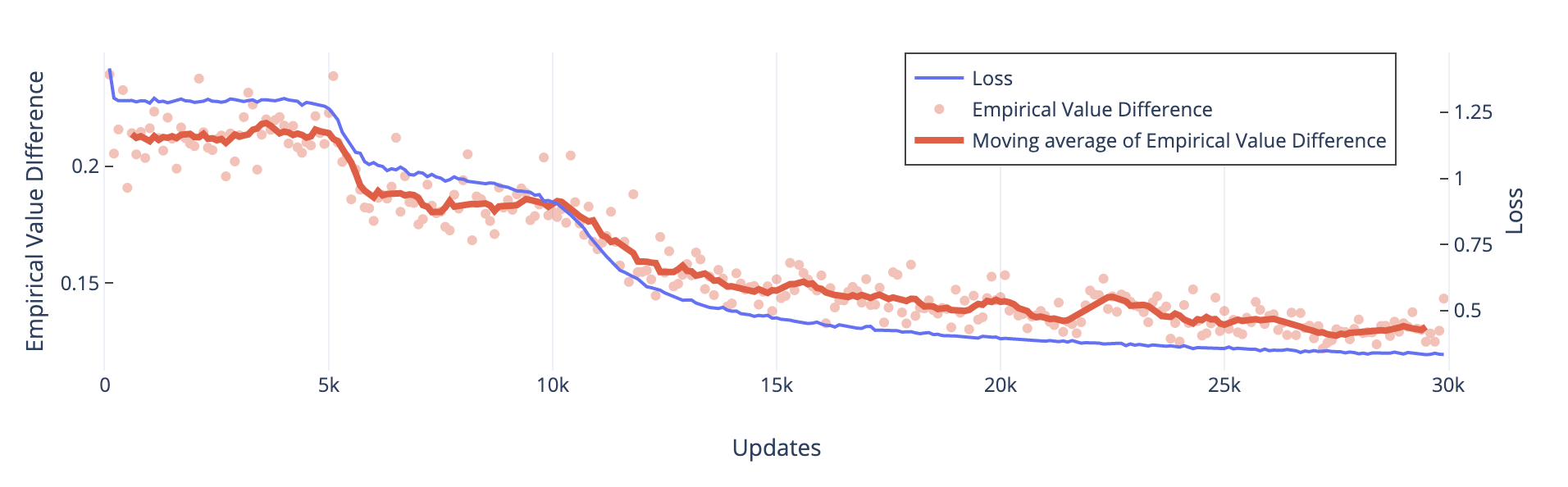}
\caption{Plot of training curves obtained by learning a DeepMDP on our toy environment. Our objective minimizes both the theoretical upper bound of value difference and the empirical value difference.}
\label{fig:toyopt}
\end{figure*}

Figure \ref{fig:toyopt} shows the loss curves for our learning procedure. We randomly sample trajectories of length 1000, and compute both the empirical reward in the real environment and the reward approximated by performing the same actions in the DeepMDP; this allows us to compute the empirical value error. These results demonstrate that neural optimization techniques are capable of learning DeepMDPs, and that this optimization procedure, designed to tighten theoretical bounds, is minimized by a good model of the environment, as reflected in improved empirical outcomes.

\section{Atari 2600 Experiments} \label{appendix:atari}

\begin{table}[h!]
    \small
    \centering
    \begin{tabular}{c c}
    Hyperparameter & Value \\
    \hline \hline
    \texttt{Runner.sticky\_actions} & \texttt{No\ Sticky\ Actions} \\
    \texttt{Runner.num\_iterations} & \texttt{200} \\
    \texttt{Runner.training\_steps} & \texttt{250000}  \\
    \texttt{Runner.evaluation\_steps} & \texttt{Eval\ phase\ not\ used.} \\
    \texttt{Runner.max\_steps\_per\_episode} & \texttt{27000} \\
    & \\
    \texttt{WrappedPrioritizedReplayBuffer.replay\_capacity} & \texttt{1000000} \\
    \texttt{WrappedPrioritizedReplayBuffer.batch\_size} & \texttt{32} \\
    & \\
    \texttt{RainbowAgent.num\_atoms} & \texttt{51} \\
    \texttt{RainbowAgent.vmax} & \texttt{10.} \\
    \texttt{RainbowAgent.update\_horizon} & \texttt{1} \\
    \texttt{RainbowAgent.min\_replay\_history} & \texttt{50000}  \\
    \texttt{RainbowAgent.update\_period} & \texttt{4} \\
    \texttt{RainbowAgent.target\_update\_period} & \texttt{10000} \\
    \texttt{RainbowAgent.epsilon\_train} & \texttt{0.01} \\
    \texttt{RainbowAgent.epsilon\_eval} & \texttt{0.001} \\
    \texttt{RainbowAgent.epsilon\_decay\_period} & \texttt{100000} \\
    \texttt{RainbowAgent.replay\_schem}e & \texttt{'uniform'} \\
    \texttt{RainbowAgent.tf\_device} & \texttt{'/gpu:0'}  \\
    \texttt{RainbowAgent.optimizer} & \texttt{@tf.train.AdamOptimizer()} \\
    & \\
    \texttt{tf.train.AdamOptimizer.learning\_rate} & \texttt{0.00025} \\
    \texttt{tf.train.AdamOptimizer.epsilon} & \texttt{0.0003125} \\
    & \\
    \texttt{ModelRainbowAgent.reward\_loss\_weight} & \texttt{1.0} \\
    \texttt{ModelRainbowAgent.transition\_loss\_weight} & \texttt{1.0} \\
    \texttt{ModelRainbowAgent.transition\_model\_type} & \texttt{'convolutional'} \\
    \texttt{ModelRainbowAgent.embedding\_type} & \texttt{'conv\_layer\_embedding'} \\
    \end{tabular}
    \caption{Configurations for the DeepMDP and C51 agents used with Dopamine \cite{bellemaredopamine} in Section \ref{ssec:c51vsdmdp}. Note that the DeepMDP is referred to as ModelRainbowAgent in the configs.}
    \label{tab:gin_configs}
\end{table}

\begin{figure*}[t]
    \centering
    \includegraphics[keepaspectratio, width=0.37\textwidth]{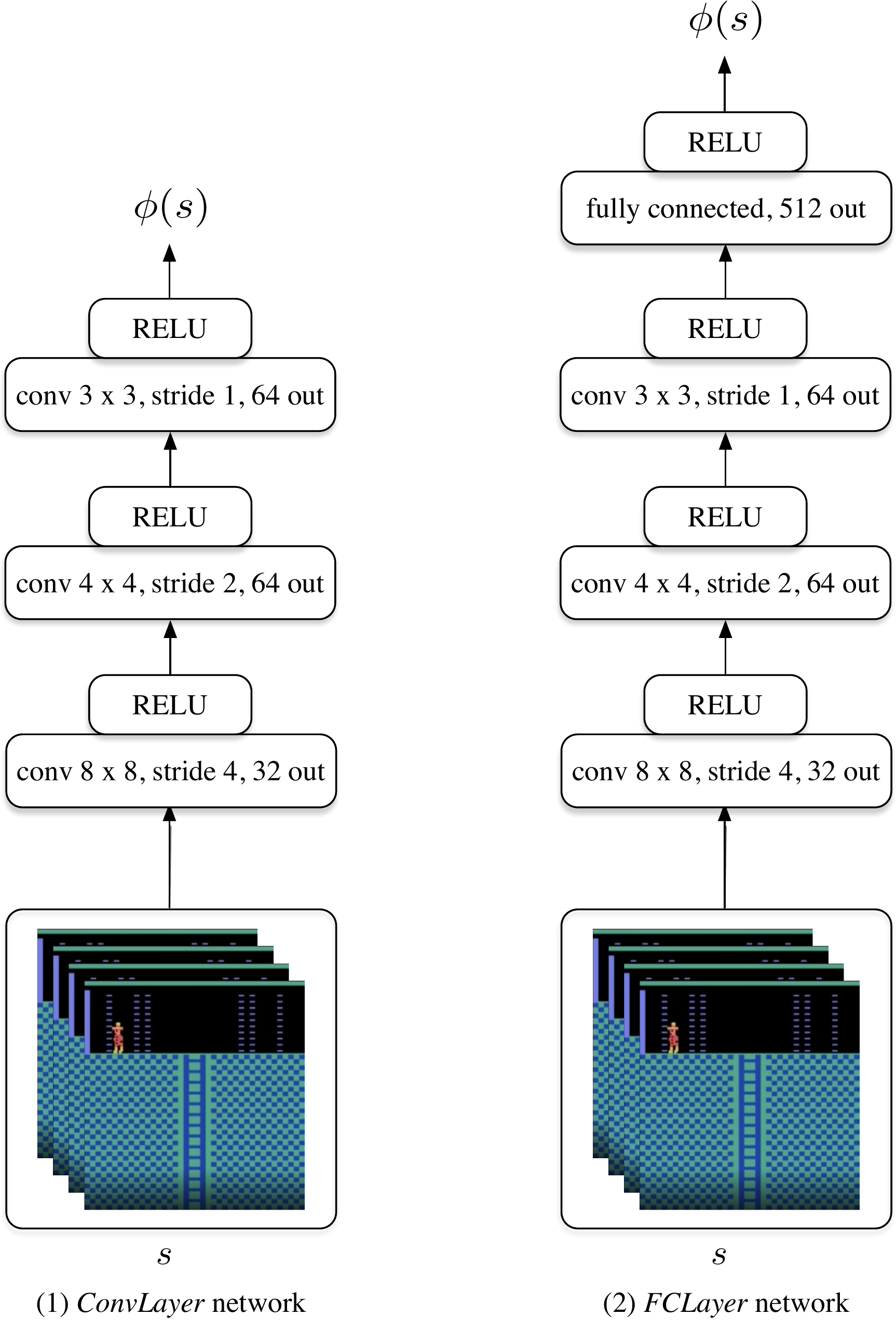}
    \caption{Encoder architectures used for the DeepMDP agent.}
    \label{fig:encoder_architectures}
\end{figure*}

\begin{figure*}[t]
    \centering
    \includegraphics[keepaspectratio, width=0.6\textwidth]{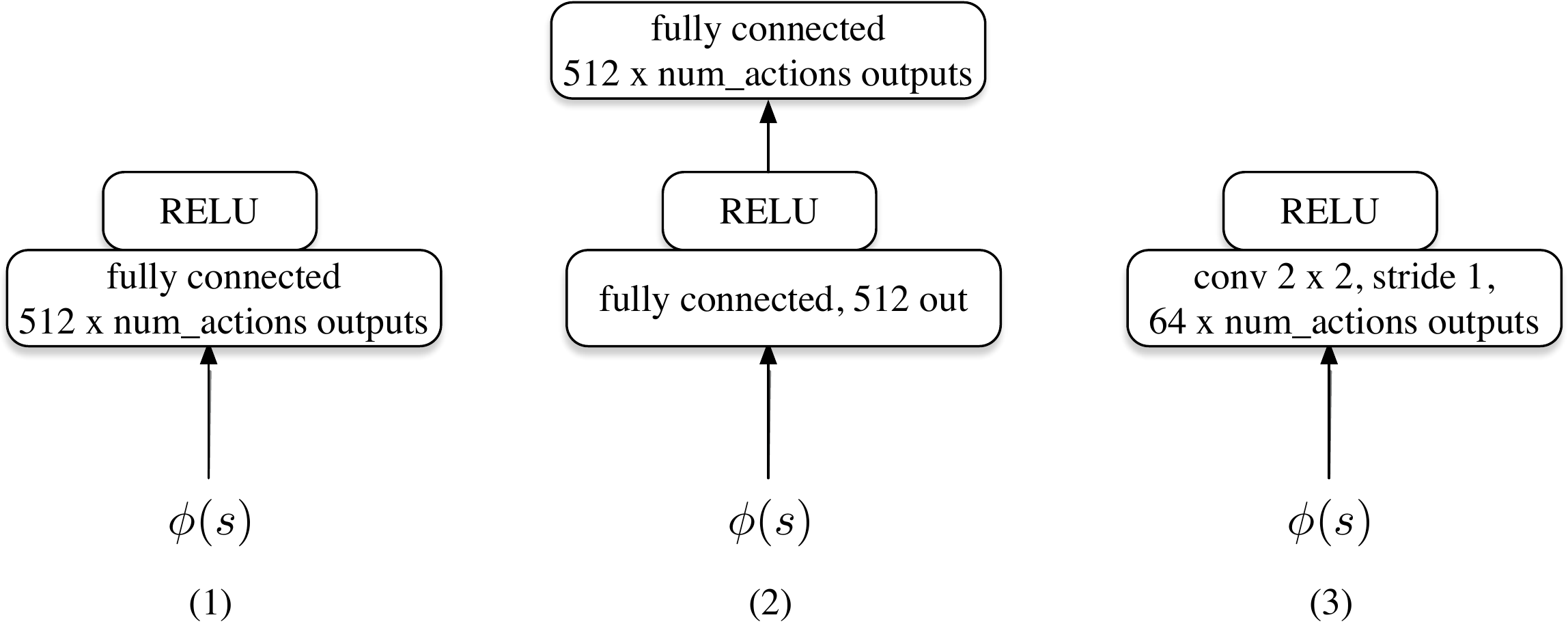}
    \caption{Transition model architectures used for the DeepMDP agent: (1) a single fully-connected layer (used with latent states of type FCLayer), (2) a two-layer fully-connected network (used with latent states of type FCLayer), and (3) a single convolutional layer (used with latent states of type ConvLayer).}
    \label{fig:transition_model_architectures}
\end{figure*}

\begin{figure*}[t]
    \centering
    \includegraphics[keepaspectratio, width=0.38\textwidth]{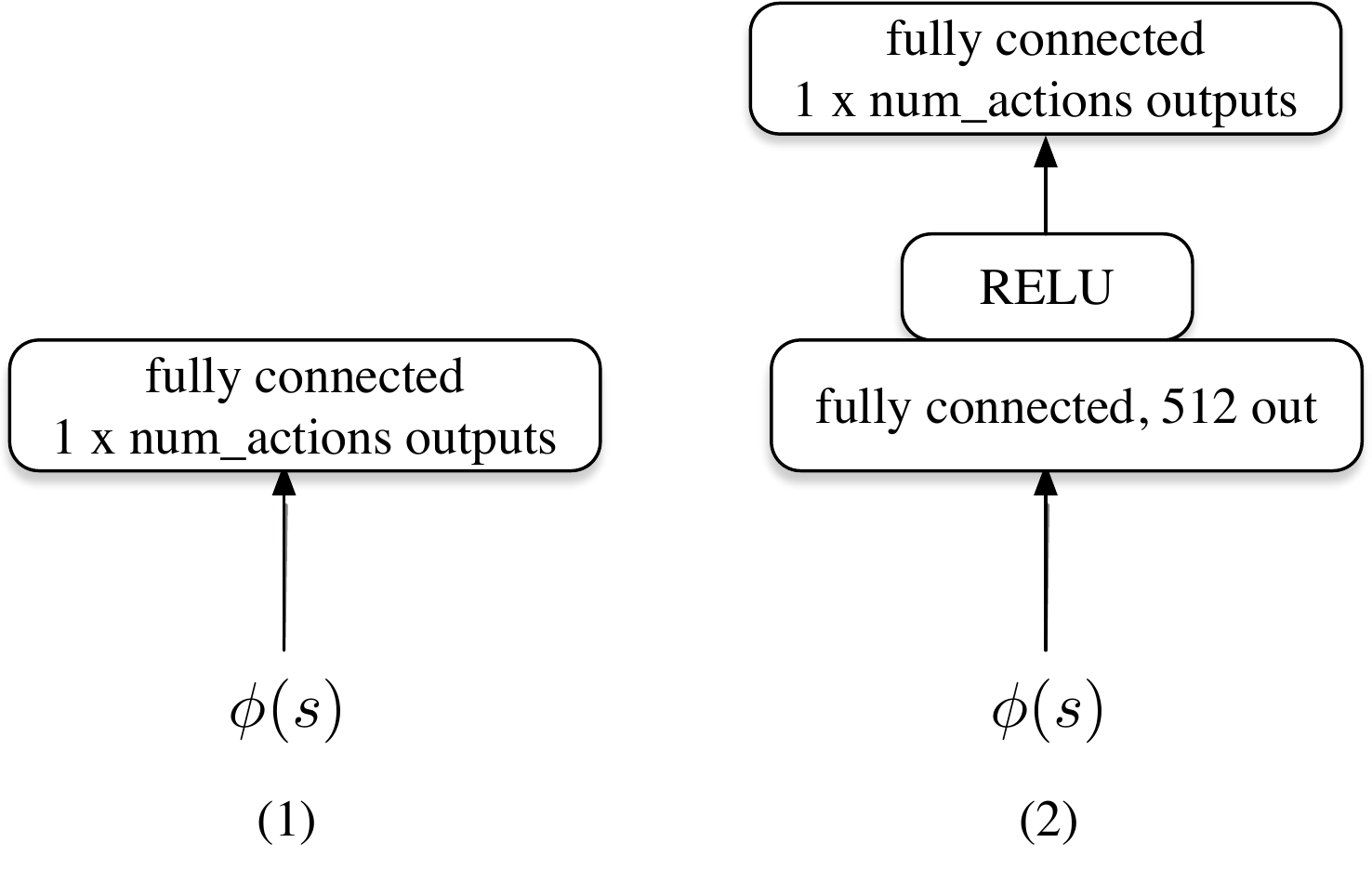}
    \caption{Reward and C51 Logits network architectures used for the DeepMDP agent: (1) a single fully-connected layer (used with latent states of type FCLayer), (2) a two-layer fully-connected network (used with latent states of type ConvLayer).}
    \label{fig:reward_model_architectures}
\end{figure*}

\begin{figure*}[t]
    \centering
    \includegraphics[keepaspectratio, width=0.4\textwidth]{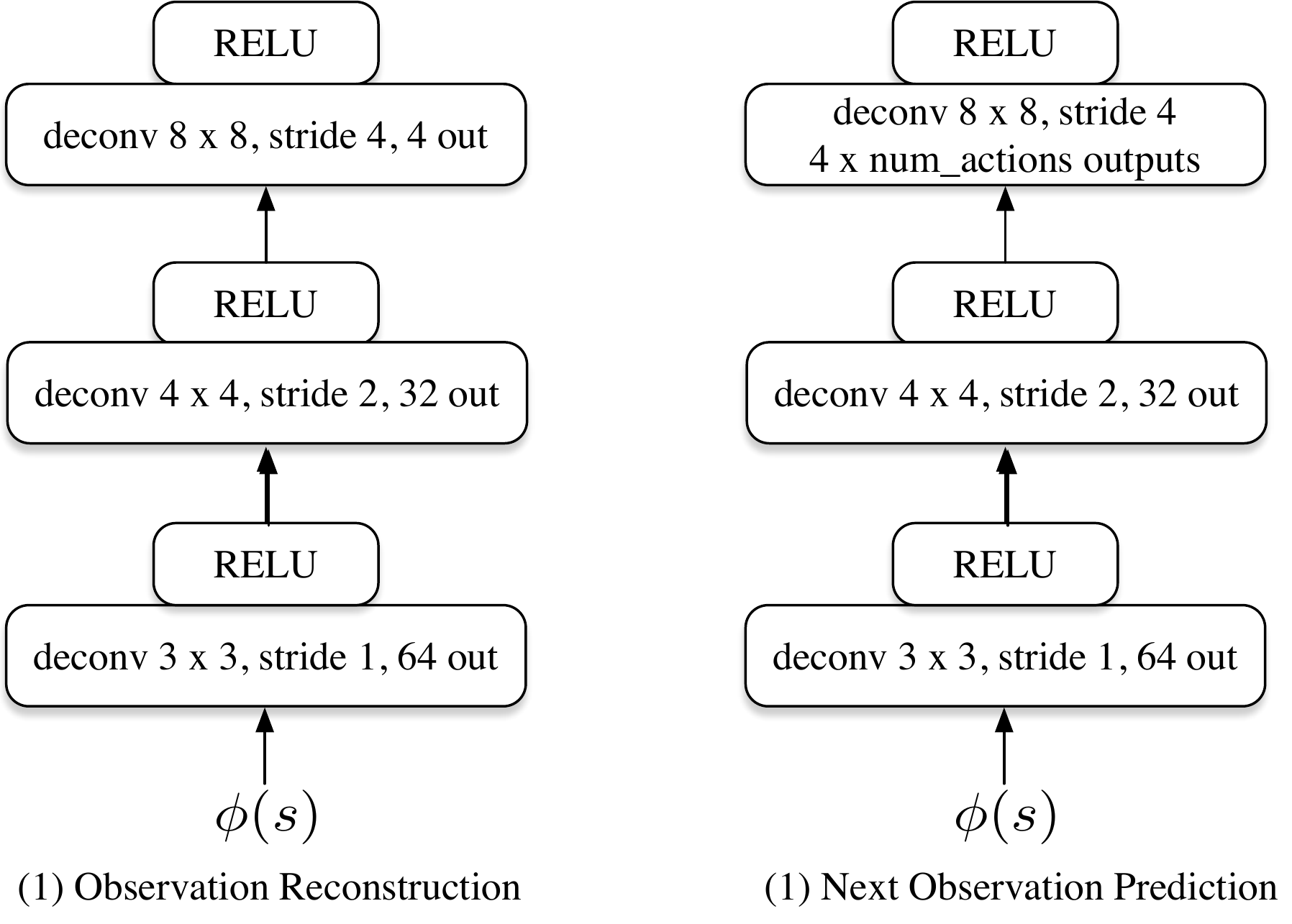}
    \caption{Architectures used for observation reconstruction and next observation prediciton. Both networks take latent states of type ConvLayer as input.}
    \label{fig:image_prediction_architectures}
\end{figure*}

\subsection{Hyperparameters}
For all experiments we use an Adam Optimizer with a learning rate of $0.00025$ and epsilon of $0.0003125$. We linearly decay epsilon from $1.0$ to $0.01$ over $1000000$ training steps. We use a replay memory of size $1000000$ (it must reach a minimum size of $50000$ prior to sampling transitions for training). Unless otherwise specified, the batch size is $32$. For additional hyperparameter details, see Table \ref{tab:gin_configs} and \cite{bellemare17distributional}. 

\subsection{Architecture Search} \label{ssec:architectures}

In this section, we aim to answer: what latent state space and transition model architecture lead to the best Atari 2600 performance of the C51 DeepMDP? We begin by jointly determining the form of $\dsspace$ and $\theta_{\bar{\prob}}$ which are conducive to learning a DeepMDP on Atari 2600 games. We employ three latent transition model architectures: (1) single fully connected layer, (2) two-layer fully-connected network, and (3) single convolutional layer. The fully-connected transition networks use the $512$-dimensional output of the embedding network's penultimate layer as the latent state, while the convolutional transition model uses the $11\times11\times64$ output of the embedding network's final convolutional layer. Empirically, we find that the use of a convolutional transition model on the final convolutional layer's output outperforms the other architectures, as shown in Figure \ref{fig:architectures}.

\subsection{Architecture Details}

The architectures of various components are described below. A conv layer refers to a 2D convolutional layer with a specified stride, kernel size, and number of outputs. A deconv layer refers to a deconvolutional layer. The padding for conv and deconv layers is such that the output layer has the same dimensionality as the input. A maxpool layer performs max-pooling on a 2D input and fully connected refers to a fully-connected layer. 

\subsubsection{Encoder}\label{ssec:encoderarchitecture}

In the main text, the encoder is referred to as $\phi : \sspace \rightarrow \dsspace$ and is parameterized by $\theta_e$. The encoder architecture is as follows:

Input: observation $s$ which has shape: $\text{batch size}\times84\times84\times4$. The Atari 2600 frames are $84\times84$ and there are $4$ stacked frames given as input. The frames are pre-processed by dividing by the maximum pixel value, $255$. Output: latent state $\phi(s)$

In Appendix \ref{ssec:architectures}, we experimented with two different latent state representations. 
(1) $\textit{ConvLayer}$: The latent state is the output of the final convolutional layer, or (2) $\textit{FCLayer}$: the latent state is the output of a fully-connected (FC) layer following the final convolutional layer. These possibilities for the encoder architecture are described in Figure \ref{fig:encoder_architectures}.

In sections \ref{ssec:c51vsdmdp}, \ref{ssec:mbrlcompare}, \ref{sec:moreauxiliarytasks}, and \ref{sec:appendix-atari-rep} the latent state of type ConvLayer is used: $11\times11\times64$ outputs of the final convolutional layer.

\subsubsection{Latent Transition Model}

In Appendix \ref{ssec:architectures} there are three types of latent transition models $\dprob: \dsspace \rightarrow \dsspace$ parameterized by $\theta_{\dprob}$ which are evaluated: (1) a single fully-connected layer, (2) a two-layer fully-connected network, and (3) a single convolutional layer (see Figure \ref{fig:transition_model_architectures}). Note that the first two types of transition models operate on the flattened $512$-dimensional latent state (FCLayer), while the convolutional transition model receives as input the $11\times11\times64$ latent state type ConvLayer. For each transition model, num$\_$actions predictions are made: one for each action conditioned on the current latent state $\phi(s)$. 

In sections \ref{ssec:c51vsdmdp}, \ref{ssec:mbrlcompare}, \ref{sec:moreauxiliarytasks}, and \ref{sec:appendix-atari-rep} the convolutional transition model is used.

\subsubsection{Reward Model and C51 Logits Network}
The architectures of the reward model $\drew$ parameterized by $\theta_{\drew}$ and C51 logits network parameterized by $\theta_Z$ depend the latent state representation. See Figure \ref{fig:reward_model_architectures} for these architectures. For each architecture type, num$\_$actions predictions are made: one for each action conditioned on the current latent state $\phi(s)$. 

In sections \ref{ssec:c51vsdmdp}, \ref{ssec:mbrlcompare}, \ref{sec:moreauxiliarytasks}, and \ref{sec:appendix-atari-rep} two-layer fully-connected networks are used for the reward and C51 logits networks.

\subsubsection{Observation Reconstruction and Next Observation Prediction}
The models for observation reconstruction and next observation prediction in Section $\ref{ssec:mbrlcompare}$ are deconvolutional networks based on the architecture of the embedding function $\phi$. Both operate on latent states of type ConvLayer. The architectures are described in Figure \ref{fig:image_prediction_architectures}.

%

\subsection{DeepMDP Auxiliary Tasks: Different Weightings on DeepMDP Losses} \label{sec:moreauxiliarytasks}
\begin{figure*}
    \centering
    \includegraphics[keepaspectratio, width=0.9\textwidth]{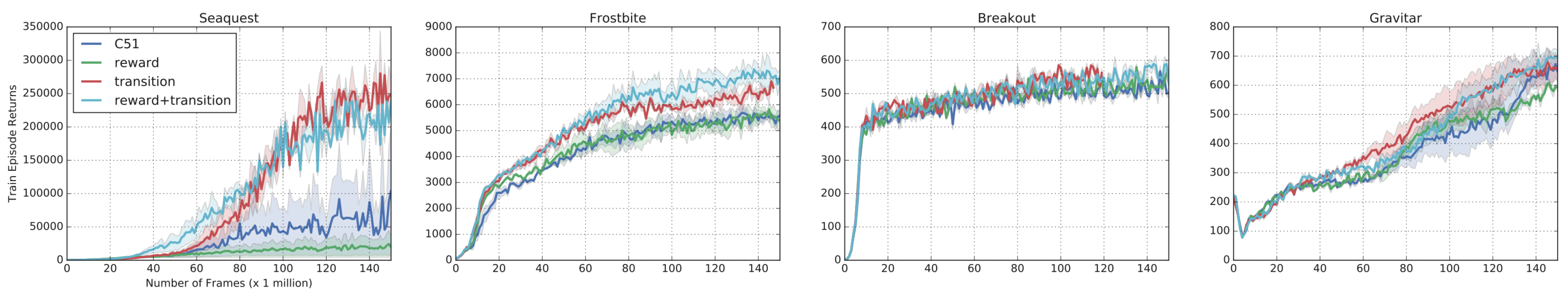}
    \caption{We compare C51 with C51 with DeepMDP auxiliary task losses. The combinations of loss weightings are $\{0, 1\}$ (just reward), $\{1, 0\}$ (just transition), and $\{1, 1\}$ (reward+transition), where the first number is the weight for the transition loss and the second number is the weight for the reward loss.}
    \label{fig:lossweightings}
\end{figure*}

In this section, we discuss results of a set of experiments where we use a convolutional latent transition model and a two-layer reward model to form auxiliary task objectives on top of a C51 agent. In these experiments, we use different weightings in the set $\{0, 1\}$ for the transition loss and for the reward loss. The network architecture is based on the best performing DeepMDP architecture in Appendix \ref{ssec:architectures}. Our results show that using the transition loss is enough to match performance of using both the transition and reward loss. In fact, on Seaquest, using only the reward loss as an auxiliary tasks causes performance to crash. See Figure \ref{fig:lossweightings} for the results.

\subsection{Representation Learning with DeepMDP Objectives}
\label{sec:appendix-atari-rep}
\begin{figure*}
    \centering
    \includegraphics[keepaspectratio, width=0.9\textwidth]{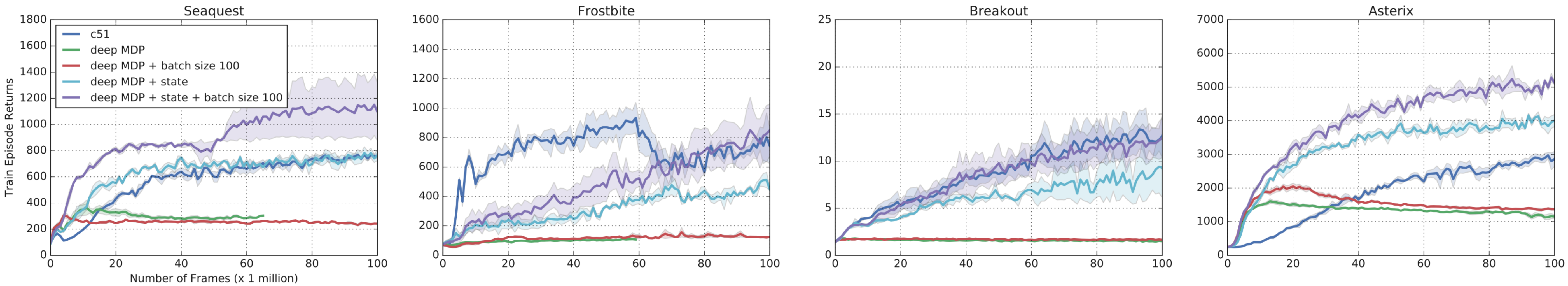}
    \caption{We evaluate the performance of C51 when learning the latent state representation only via minimizing deepMDP objectives. We compare learning the latent state representation with the deepMDP objectives (deep MDP), deepMDP objectives with larger batch sizes (deepMDP + batch size 100), deepMDP objectives and an observation reconstruction loss (deepMDP + state), and deepMDP with both a reconstruction loss and larger batch size (deepMDP + state + batch size 100). As a baselines, we compare to C51 on a random latent state representation (C51).}
    \label{fig:dmdp_representation}
\end{figure*}

Given performance improvements in the auxiliary task setting, a natural question is whether optimization of the deepMDP losses is sufficient to perform model-free RL. To address this question, we learn $\theta_e$ only via minimizing the reward and latent transition losses. We then learn $\theta_Z$ by minimizing the C51 loss but do not pass gradients through $\theta_e$. As a baseline, we minimize the C51 loss with randomly initialized $\theta_e$ and do not update $\theta_e$. In order to successfully predict terminal transitions and rewards, we add a terminal reward loss and a terminal state transition loss. The terminal reward loss is a Huber loss between $\drew(\phi(s_T))$ and $0$, where $s_T$ is a terminal state. The terminal transition loss is a Huber loss between $\dprob(s,a)$ and $\textbf{0}$, where $s$ is either a terminal state or a state immediately preceding a terminal state and $\textbf{0}$ is the zero latent state.

We find that in practice, minimizing the latent transition loss causes the latent states to collapse to $\phi(s) = 0$ $\forall s \in S$. As \cite{lavet2018crar} notes, if only the latent transition loss was minimized, then the optimal solution is indeed $\phi : S \rightarrow 0$ so that $\dprob$ perfectly predicts $\phi(\prob(s,a))$. 

We hope to mitigate representation collapse by augmenting the influence of the reward loss. We increase the batch size from $32$ to $100$ to acquire greater diversity of rewards in each batch sampled from the replay buffer. However, we find that only after introducing a state reconstruction loss do we obtain performance levels on par with our simple baseline. These results (see Figure \ref{fig:dmdp_representation}) indicate that in more complex environments, additional work is required to successfully balance the minimization of the transition loss and the reward loss, as the transition loss seems to dominate. 

This finding was surprising, since we were able to train a DeepMDP on the DonutWorld environment with no reconstruction loss. Further investigation of the DonutWorld experiments shows that the DeepMDP optimization procedure seems to be highly prone to becoming trapped in local minima. The reward loss encourages latent states to be informative, but the transition loss counteracts this, preferring latent states which are uninformative and thus easily predictable. Looking at the relative reward and transition losses in early phases of training in Figure \ref{fig:competeloss}, we see this dynamic clearly. At the start of training, the transition loss quickly forces latent states to be near-zero, resulting in very high reward loss. Eventually, on this simple task, the model is able to escape this local minimum by ``discovering'' a representation that is both informative and predictable. However, as the difficulty of a task scales up, it becomes increasingly difficult to discover a representation which escapes these local minima by explaining the underlying dynamics of the environment well. This explains our observations on the Arcade Learning Environment; the additional supervision from the reconstruction loss helps guide the algorithm towards representations which explain the environment well.

\begin{figure*}
    \centering
    \includegraphics[keepaspectratio, width=0.9\textwidth]{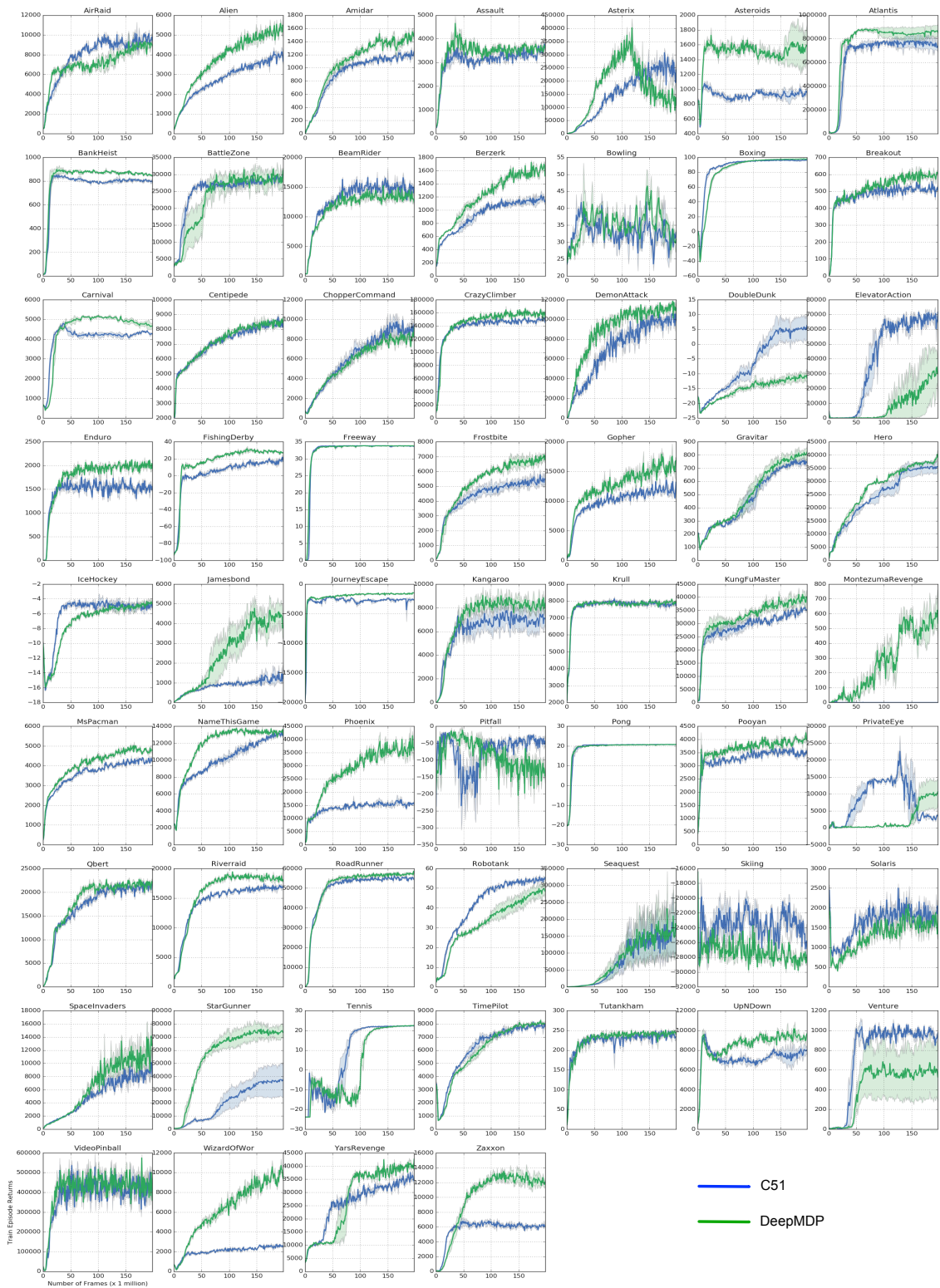}
    \caption{Learning curves of C51 and C51 + DeepMDP auxiliary task objectives (labeled DeepMDP) on Atari 2600 games.}
    \label{fig:c51_vs_dmdp_60_games}
\end{figure*}

\begin{table}[h!]
  \begin{center}
    \label{tab:DeepMDP60games}
    \centering
    \small
    \begin{tabular}{c c c}
        \text{Game Name} & \text{C51} & \text{DeepMDP} \\
        \hline
        \hline
        \texttt{AirRaid}&	\textbf{11544.2}&	10274.2\\
        \texttt{Alien}&	4338.3&	\textbf{6160.7}\\
        \texttt{Amidar}&	1304.7&	\textbf{1663.8}\\
        \texttt{Assault}&	4133.4&	\textbf{5026.2}\\
        \texttt{Asterix}&	343210.0&	\textbf{452712.7}\\
        \texttt{Asteroids}&	1125.4&	\textbf{1981.7}\\
        \texttt{Atlantis}&	844063.3&	\textbf{906196.7}\\
        \texttt{BankHeist}&	861.3&	\textbf{937.0}\\
        \texttt{BattleZone}&	31078.2&	\textbf{34310.2}\\
        \texttt{BeamRider}&	\textbf{19081.0}&	16216.8\\
        \texttt{Berzerk}&	1250.9&	\textbf{1799.9}\\
        \texttt{Bowling}&	51.4&	\textbf{56.3}\\
        \texttt{Boxing}&	97.3&	\textbf{98.2}\\
        \texttt{Breakout}&	584.1&	\textbf{672.8}\\
        \texttt{Carnival}&	4877.3&	\textbf{5319.8}\\
        \texttt{Centipede}&	\textbf{9092.1}&	9060.9\\
        \texttt{ChopperCommand}&	\textbf{10558.8}&	9895.7\\
        \texttt{CrazyClimber}&	158427.7&	\textbf{173043.1}\\
        \texttt{DemonAttack}&	111697.7&	\textbf{119224.7}\\
        \texttt{DoubleDunk}&	\textbf{6.7}&	-9.3\\
        \texttt{ElevatorAction}&	\textbf{73943.3}&	37854.4\\
        \texttt{Enduro}&	1905.3&	\textbf{2197.8}\\
        \texttt{FishingDerby}&	25.4&	\textbf{33.9}\\
        \texttt{Freeway}&	\textbf{33.9}&	33.9\\
        \texttt{Frostbite}&	5882.9&	\textbf{7367.3}\\
        \texttt{Gopher}&	15214.3&	\textbf{21017.2}\\
        \texttt{Gravitar}&	790.4&	\textbf{838.3}\\
        \texttt{Hero}&	36420.7&	\textbf{40563.1}\\
        \texttt{IceHockey}&	\textbf{-3.5}&	-4.1\\
        \texttt{Jamesbond}&	1776.7&	\textbf{5181.1}\\
        \texttt{JourneyEscape}&	-1856.1&	\textbf{-1337.1}\\
        \texttt{Kangaroo}&	8815.5&	\textbf{9714.9}\\
        \texttt{Krull}&	8201.5&	\textbf{8246.9}\\
        \texttt{KungFuMaster}&	37956.5&	\textbf{42692.7}\\
        \texttt{MontezumaRevenge}&	14.7&	\textbf{770.7}\\
        \texttt{MsPacman}&	4597.8&	\textbf{5282.5}\\
        \texttt{NameThisGame}&	13738.7&	\textbf{14064.6}\\
        \texttt{Phoenix}&	20216.7&	\textbf{45565.1}\\
        \texttt{Pitfall}&	-9.8&	\textbf{-0.8}\\
        \texttt{Pong}&	\textbf{20.8}&	20.8\\
        \texttt{Pooyan}&	4052.7&	\textbf{4431.1}\\
        \texttt{PrivateEye}&	\textbf{28694.0}&	11223.8\\
        \texttt{Qbert}&	23268.6&	\textbf{23538.7}\\
        \texttt{Riverraid}&	17845.1&	\textbf{19934.7}\\
        \texttt{RoadRunner}&	57638.5&	\textbf{59152.2}\\
        \texttt{Robotank}&	\textbf{57.4}&	51.3\\
        \texttt{Seaquest}&	226264.0&	\textbf{230881.6}\\
        \texttt{Skiing}&	\textbf{-15454.8}&	-16478.0\\
        \texttt{Solaris}&	\textbf{2876.7}&	2506.8\\
        \texttt{SpaceInvaders}&	12145.8&	\textbf{16461.2}\\
        \texttt{StarGunner}&	38928.7&	\textbf{78847.6}\\
        \texttt{Tennis}&	22.6&	\textbf{22.7}\\
        \texttt{TimePilot}&	8340.7&	\textbf{8345.6}\\
        \texttt{Tutankham}&	\textbf{259.3}&	256.9\\
        \texttt{UpNDown}&	10175.5&	\textbf{10930.6}\\
        \texttt{Venture}&	\textbf{1190.1}&	755.4\\
        \texttt{VideoPinball}&	\textbf{668415.7}&	633848.8\\
        \texttt{WizardOfWor}&	2926.0&	\textbf{11846.1}\\
        \texttt{YarsRevenge}&	39502.9&	\textbf{44317.8}\\
        \texttt{Zaxxon}&	7436.5&	\textbf{14723.0}\\
    \end{tabular}
    \caption{DeepMDP versus C51 returns. For both agents, we report the max average score achieved across all training iterations (each training iteration is 1 million frames).}
  \end{center}
\end{table}
}

\end{document}